\pdfoutput=1  
\documentclass{article}
\usepackage[main, final]{neurips_2026} 

\newif\ifchecklist \checklisttrue

\usepackage[T1]{fontenc}
\usepackage{amsmath, amssymb, amsthm}
\usepackage{hyperref} 
\usepackage{url}
\usepackage{booktabs}
\usepackage{xcolor}
\usepackage{bm}
\usepackage{multirow}
\usepackage{enumitem}
\usepackage{graphicx}
\usepackage{algorithm}
\usepackage{algorithmic}
\usepackage{subcaption}
\newtheorem{proposition}{Proposition}
\newtheorem{remark}{Remark}

\NewDocumentCommand \prArg{mm}
{(
\IfNoValueTF{#2}{#1}{#1 \mid #2}
)}

\NewDocumentCommand \newProbabilityFormat{r<>m}
{
	\NewDocumentCommand #1 {e{_}e{^}>{\SplitArgument{1}{|}}d()}
	{
		\IfNoValueTF{##1}
		{
			\IfNoValueTF{##2}
			{\IfNoValueTF{##3}{#2}{#2\prArg##3}}
			{\IfNoValueTF{##3}{#2^{##2}}{#2^{##2}\prArg##3}}
		}
		{
			\IfNoValueTF{##2}
			{\IfNoValueTF{##3}{#2_{##1}}{#2_{##1}\prArg##3}}
			{\IfNoValueTF{##3}{#2_{##1}^{##2}}{#2_{##1}^{##2}\prArg##3}}
		}
	}
}

\NewDocumentCommand \fProbability {m} {#1}
\NewDocumentCommand \fVector {m} {\boldsymbol{#1}}
\NewDocumentCommand \fMatrix {m} {\boldsymbol{#1}}

\NewDocumentCommand \fFunction {m} {{#1}}
\NewDocumentCommand \fSet {m} {\mathcal{#1}}

\NewDocumentCommand \newScalar{r<>m}
{
	\NewDocumentCommand #1 {} {{#2}}
}

\NewDocumentCommand \newVector{r<>m}
{
	\NewDocumentCommand #1 {} {\fVector{#2}}
}

\NewDocumentCommand \newMatrix{r<>m}
{
	\NewDocumentCommand #1 {} {\fMatrix{#2}}
}

\NewDocumentCommand \newProbability{r<>m}
{
	\NewDocumentCommand #1 {} {\fProbability{#2}}
}

\NewDocumentCommand \newFunction{r<>m}
{
	\NewDocumentCommand #1 {} {\fFunction{#2}}
}

\NewDocumentCommand \newSet{r<>m}
{
	\NewDocumentCommand #1 {} {\fSet{#2}}
}

\newScalar<\eps>{\varepsilon}

\newVector<\va>{a}
\newVector<\vb>{b}
\newVector<\vc>{c}
\newVector<\vd>{d}
\newVector<\ve>{e}
\newVector<\vf>{f}
\newVector<\vg>{g}
\newVector<\vh>{h}
\newVector<\vi>{i}
\newVector<\vj>{j}
\newVector<\vk>{k}
\newVector<\vl>{\ell}
\newVector<\vm>{m}
\newVector<\vn>{n}
\newVector<\vo>{o}
\newVector<\vp>{p}
\newVector<\vq>{q}
\newVector<\vr>{r}
\newVector<\vs>{s}
\newVector<\vt>{t}
\newVector<\vu>{u}
\newVector<\vv>{v}
\newVector<\vw>{w}
\newVector<\vx>{x}
\newVector<\vy>{y}
\newVector<\vz>{z}
\newVector<\vAlpha>{\alpha}
\newVector<\vBeta>{\beta}
\newVector<\vGamma>{\gamma}
\newVector<\vDelta>{\delta}
\newVector<\vEpsilon>{\varepsilon}
\newVector<\vEps>{\varepsilon}
\newVector<\vZeta>{\zeta}
\newVector<\vEta>{\eta}
\newVector<\vTheta>{\theta}
\newVector<\vTh>{\theta}
\newVector<\vKappa>{\kappa}
\newVector<\vLambda>{\lambda}
\newVector<\vMu>{\mu}
\newVector<\vNu>{\nu}
\newVector<\vXi>{\xi}
\newVector<\vPi>{\pi}
\newVector<\vRho>{\rho}
\newVector<\vSigma>{\sigma}
\newVector<\vTau>{\tau}
\newVector<\vUpsilon>{\upsilon}
\newVector<\vPhi>{\varphi}
\newVector<\vChi>{\chi}
\newVector<\vPsi>{\psi}
\newVector<\vOmega>{\omega}

\newMatrix<\mA>{A}
\newMatrix<\mB>{B}
\newMatrix<\mC>{C}
\newMatrix<\mD>{D}
\newMatrix<\mE>{E}
\newMatrix<\mF>{F}
\newMatrix<\mG>{G}
\newMatrix<\mH>{H}
\newMatrix<\mI>{I}
\newMatrix<\mJ>{J}
\newMatrix<\mK>{K}
\newMatrix<\mL>{L}
\newMatrix<\mM>{M}
\newMatrix<\mN>{N}
\newMatrix<\mP>{P}
\newMatrix<\mQ>{Q}
\newMatrix<\mR>{R}
\newMatrix<\mS>{S}
\newMatrix<\mT>{T}
\newMatrix<\mU>{U}
\newMatrix<\mV>{V}
\newMatrix<\mW>{W}
\newMatrix<\mX>{X}
\newMatrix<\mY>{Y}
\newMatrix<\mZ>{Z}
\newMatrix<\mGamma>{\Gamma}
\newMatrix<\mDelta>{\Delta}
\newMatrix<\mTheta>{\Theta}
\newMatrix<\mLambda>{\Lambda}
\newMatrix<\mXi>{\Xi}
\newMatrix<\mPi>{\Pi}
\newMatrix<\mSigma>{\Sigma}
\newMatrix<\mUpsilon>{\Upsilon}
\newMatrix<\mPhi>{\Phi}
\newMatrix<\mPsi>{\Psi}
\newMatrix<\mOmega>{\Omega}

\newSet<\sA>{A}
\newSet<\sB>{B}
\newSet<\sC>{C}
\newSet<\sD>{D}
\newSet<\sE>{E}
\newSet<\sF>{F}
\newSet<\sG>{G}
\newSet<\sH>{H}
\newSet<\sI>{I}
\newSet<\sJ>{J}
\newSet<\sK>{K}
\newSet<\sL>{L}
\newSet<\sM>{M}
\newSet<\sN>{N}
\newSet<\sP>{P}
\newSet<\sQ>{Q}
\newSet<\sR>{R}
\newSet<\sS>{S}
\newSet<\sT>{T}
\newSet<\sU>{U}
\newSet<\sV>{V}
\newSet<\sW>{W}
\newSet<\sX>{X}
\newSet<\sY>{Y}
\newSet<\sZ>{Z}

\title{Dense Structural Compression of Transformers via Gauge-Correct Channel Removal}
\author{
Jed A. Duersch\thanks{Corresponding author.} \quad
Na\"im Es-Sebbani \quad
Nathana\"el Haas \quad
Zied Bouraoui \\
Universit\'e d'Artois, CNRS, CRIL UMR 8188, Lens, France
}

\newcommand{\ffrelSeeds}{25}

\newcommand{\ffrelTarget}{0.005}
\newcommand{\ffrelInitialFMA}{13{,}179{,}850}

\newcommand{\ffrelPubConverged}{18}
\newcommand{\ffrelPubNeverComp}{2}
\newcommand{\ffrelPubDegraded}{5}

\newcommand{\ffrelPubCompression}{187}

\newcommand{\ffrelPubReachMedian}{2400}
\newcommand{\ffrelPubReachMin}{1700}
\newcommand{\ffrelPubReachMax}{17100}
\newcommand{\ffrelPubNeverReached}{2}

\newcommand{\ffrelHalfConverged}{20}

\newcommand{\ffrelHalfDegraded}{5}

\newcommand{\ffrelHalfCompression}{176}

\newcommand{\ffrelHalfReachMedian}{1700}
\newcommand{\ffrelHalfReachMin}{1200}
\newcommand{\ffrelHalfReachMax}{3300}

\newcommand{\ffrelRateFisherP}{0.74}
\newcommand{\ffrelDegradedFisherP}{1.00}

\newcommand{\ffrelReachP}{< 0.001}

\newcommand{\ffrelConvFMAP}{0.059}

\newcommand{\perlayerCeilA}{$D{=}576$, $H{=}6$, $d_k{=}48$, $d_v{=}96$, $d_f{=}1536$; discovered $D{=}571$}
\newcommand{\perlayerCeilB}{$D{=}480$, $H{=}8$, $d_k{=}45$, $d_v{=}90$, $d_f{=}1920$; discovered $D{=}480$}

\begin{document}
\maketitle

\begin{abstract}
Inference energy per token drives the cost and carbon footprint of deployed transformers. It is dominated by dense matrix products that incur fused multiply-accumulate (FMA) operations and memory traffic. To reduce these computations while retaining dense tensors for high GPU throughput, we develop a methodology from first principles to adapt structural complexity during training to maximize inference utility per unit compute. Channel penalties drive entire tensor slices to zero to enable physical removal while preserving density and the network function.
The natural approach, penalizing the norm of operator components acting through each channel, is provably destabilized by gauge freedom. We resolve this pathology with GaugeLasso: additive symmetric group-lasso penalties that recover a monotone function of product-norms when the network converges to gauge balance. Our equilibrium analysis enables per-channel calibration to correctly suppress slices that under-perform in inference utility per unit compute.
Under adaptive pressure, the network reorganizes into depth-dependent structural profiles that can be far smaller than the architecture required to learn the task. On polynomial long division over $\mathbb{F}_{31}$, compute compresses from 148 to 255 times with perfect accuracy. On character-level language modeling, compressed models outperform the hand-designed baseline at equal FMA. On masked autoencoding, a compression trial exposes which axes were over-provisioned and which saturated, guiding a better second design. Compaction also accelerates training monotonically (final steps 1.6 to 5 times faster) as the model progresses. Post-hoc pruning with the same utility ranking cannot reach these structures, showing that sustained pressure is central to discovery of efficient models. Retraining a discovered architecture recovers baseline quality on our statistical tasks, but fails on our exact algorithmic task.
\end{abstract}


\section{Introduction}
\label{sec:intro}

The energy required to process a token drives both the environmental footprint and operating costs of a deployed transformer.
Nearly all of that energy is spent in dense matrix-matrix products.
Attention and MLP projections use one fused multiply-accumulate (FMA) per parameter
per token (two FLOPs by convention),
and attention adds a further term that grows with context length.
Reducing inference energy means reducing this arithmetic, and the memory traffic
that feeds it, while preserving the predictable access patterns that make dense GPU kernels efficient.

Transformers used in practice simplify design by using structurally uniform layers
that contain equivalent processing dimensions.
Although a block at depth~1 may serve a fundamentally different role than a block at depth~12,
both are allocated identical resources
because the space of heterogeneous depth-dependent architectures is far too large to search by enumeration.
Moreover, the structure needed to discover a solution can be far larger
than the structure needed to perform inference,
leaving the trained network wasteful.

We develop a methodology to dynamically adapt network structure during training,
compelling operators to reorganize into smaller dense tensors that retain full GPU throughput.%
\footnote{Code: \url{https://gitlab.com/jduersch/gaugelasso}}
Calibrated channel-level penalties drive entire tensor slices to zero on both sides of every structural axis,
enabling physical removal of channels without introducing sparsity.
We price each channel in inference utility per FMA, rather than just parameters,
to account for the additional compute used by key and value channels.
Further, FMA is a kernel- and hardware-independent model property
that we can optimize at design time without referencing a deployment target.
Prior approaches to efficient architectures, including post-hoc structured pruning, neural architecture search,
and during-training methods such as CoFi~\citep{xia2022structured},
either operate on fixed representations, search discrete candidate sets,
or handle multiplicative coupling between connected factors heuristically.

The natural approach is to penalize the operator norm acting through each channel, but this is provably pathological.
Transformers contain \emph{gauge-connected} tensors:
rescaling one factor and inversely rescaling the other preserves the network function.
This freedom destabilizes product-norm penalties,
driving one factor toward zero while the other compensates until updates become too sensitive.
GaugeLasso resolves this problem with efficiency-scaled additive symmetric penalties
on channel-aligned factors, which drive the network toward a particular gauge balance
where the group norms measure inference utility per FMA.
This suppresses channels with under-performing utility per cost.
Applied to diverse tasks, GaugeLasso discovers architectures
far smaller than those needed for learning.
Figure~\ref{fig:intro} traces this process on polynomial long division over~$\mathbb{F}_{31}$,
where layers adopt markedly different dimensions during compression.
The method achieves up to $255{\times}$ compression (median $187{\times}$) over seeds that maintain perfect accuracy.
For character-level language modeling,
compressed models outperform the hand-designed baseline at equal FMA,
while continuous dense compaction accelerates training by up to~$5{\times}$.

Structural regularization during training allows the network to adopt more efficient representations
as it is forced into fewer dimensions, rather than merely removing redundancies.
We test this directly by applying the same gauge-calibrated utility metric to trim fully trained models.
On~$\mathbb{F}_{31}$, post-hoc pruning cannot compress the model at all;
on character-level language modeling, it achieves 22\% worse perplexity at equal FMA.
A matched comparison against CoFi, the closest during-training method,
stalls at $1.97{\times}$ compression where $7.09{\times}$ was requested,
plausibly because its mask space cannot narrow dimensions inside units that remain alive
(Section~\ref{sec:posthoc}).
Sustained resource pressure is therefore what discovers these architectures.
Whether it is additionally required to \emph{train} them depends on the task:
retraining a discovered architecture from fresh initialization
recovers baseline quality on our statistical tasks, but not on exact algorithmic execution.

Our contributions are the following.
\begin{enumerate}[leftmargin=*,itemsep=2pt]
\item \textbf{Symmetric group-lasso equilibrium and inference utility.}
At gauge balance, group norms directly measure each channel's inference utility.
Efficiency-scaled calibration normalizes to utility per FMA,
enabling gradual suppression of under-performing channels across structural axes (Section~\ref{sec:theory}).

\item \textbf{Dense adaptive training framework.}
Our implementation techniques
(adaptive penalty strength from a loss target,
gauge-Sinkhorn iterations to accelerate balance,
and optimizer-preserving tensor surgery)
yield continuous dense compression without sparsity masks, discrete search, or optimization restarts.
Because compressed dimensions are physically removed, step time decreases as the model compresses (Section~\ref{sec:training}).

\item \textbf{Sustained pressure discovers architectures that search does not reach.}
Transformers reorganize into smaller dense architectures that match or exceed baselines,
and neither post-hoc pruning with the same utility metric nor a matched
during-training baseline reaches them.
Retraining from fresh initialization recovers baseline quality on our statistical
tasks but fails on exact algorithmic execution.
Structural profiles also reveal when axes saturate their initial budget,
directly informing architecture design (Section~\ref{sec:experiments}).
\end{enumerate}

\begin{figure}[t]
\centering
{\small\textbf{Polynomial long division over $\mathbb{F}_{31}$: structural profiles and compression trajectory.}}
\includegraphics[width=\textwidth]{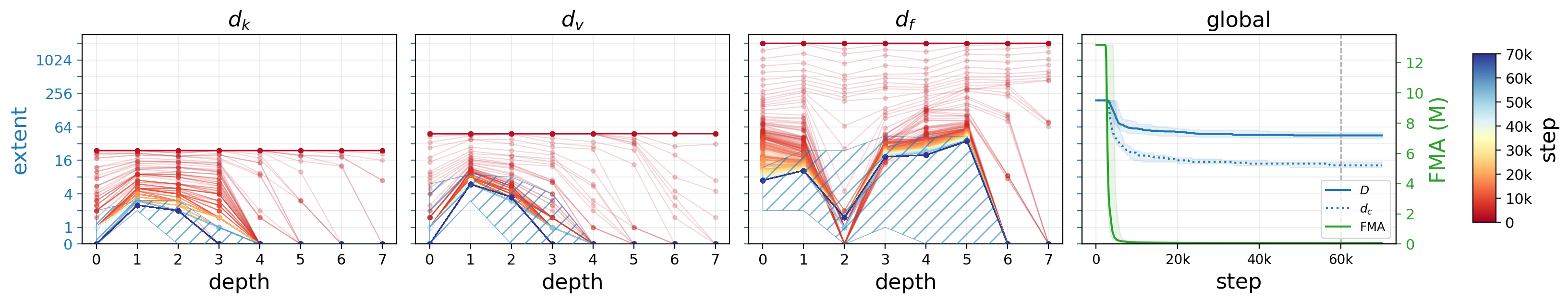}
\caption{An 8-layer transformer trained on polynomial long division over~$\mathbb{F}_{31}$
discovers specialized depth structure under compute pressure
($L^*{=}0.005$, 25~seeds, 18 converged; see Section~\ref{sec:exp_ff31}).
Median trends are shown in 3 of 7 adaptive structural axes. See Section~\ref{sec:setting} for details and full profiles in Appendix~\ref{app:exp_details}.
The hatched band spans the middle 78\%.
The right panel tracks the embedding dimension~$D$ and classifier input~$d_c$ along with inference compute.}
\label{fig:intro}
\end{figure}
\section{Related Work}
\label{sec:related}

\paragraph{Structured pruning.}
Post-hoc methods follow a train-prune-finetune pipeline:
magnitude pruning~\citep{han2015learning}, structured extensions
to filters and channels~\citep{li2017pruning,
molchanov2019importance}, and the lottery ticket
hypothesis~\citep{frankle2019lottery}.
Prior work has pruned transformers at multiple granularities,
including entire layers~\citep{fan2019reducing, sajjad2023effect},
attention heads~\citep{michel2019sixteen}, and intermediate
dimensions~\citep{hou2020dynabert}, each requiring a
dedicated mechanism.
Our framework achieves all of these through a single
efficiency-calibrated penalty: when all channels on an axis
are dropped, the owning operator or block is removed
automatically.
For large language models,
Sheared LLaMA~\citep{xia2023sheared} combines structured pruning
with continued pre-training but requires the target architecture
to be specified in advance;
SliceGPT~\citep{ashkboos2024slicegpt} uses computational
invariances to delete rows and columns post-hoc, producing
dense output like GaugeLasso but operating on fixed
representations.

During-training methods apply sparsity pressure throughout
optimization.
Movement pruning~\citep{sanh2020movement} learns which weights to
prune from gradient direction;
gradual magnitude pruning~\citep{zhu2017prune} follows a
sparsity schedule.
CoFi~\citep{xia2022structured} is the closest prior work: it
prunes transformers at multiple granularities using learned masks,
managing the resulting mask-weight coupling with straight-through
estimation and Lagrangian scheduling.
This coupling is the same gauge freedom we analyze in
Section~\ref{sec:theory};
CoFi handles it heuristically, while our additive penalty
resolves it analytically, yielding calibrated per-channel utility
at equilibrium.
CoFi maintains separate mask variables per granularity,
whereas our penalties handle all structural axes and
produce fully dense tensors rather than sparse masks.
Knowledge distillation~\citep{sanh2019distilbert, jiao2020tinybert}
is complementary: it requires a target architecture, which our
framework can discover.

\paragraph{Group lasso and scale invariance.}
Group lasso~\citep{yuan2006model} provides coordinated
suppression of parameter groups;
\citet{wen2016learning} applied it to deep networks.
The rescaling invariance between multiplicatively-coupled layers
is well known:
\citet{dinh2017sharp} showed that sharp minima
become flat under rescaling in batch-normalized networks, and
path-norm regularization~\citep{neyshabur2015norm}
attempted to regularize in a scale-invariant manner.
Equi-normalization~\citep{stock2019equi} is closely related,
using Sinkhorn-Knopp iterations to minimize the $L_2$ norm of the weights to improve training.
In contrast, our gauge Sinkhorn balances a weighted group norm,
so the balance point encodes compute cost rather than weight magnitude.
Having proved that gauge invariance destabilizes
\emph{every} monotone penalty on the product norm
(Appendix~\ref{app:failure_modes}), we resolve the pathology
constructively: additive group-lasso penalties drive the network to a
stable gauge equilibrium, at which the group norms are gauge-invariant
and therefore support consistent utility estimates, efficiency-scaled
calibration, and a drop criterion that requires no loss extrapolation.

\paragraph{Architecture search and scaling.}
NAS~\citep{zoph2017nas} and differentiable variants like
DARTS~\citep{liu2019darts} search over discrete architectural
candidates.
DynaBERT~\citep{hou2020dynabert} trains sub-networks of
varying width and depth simultaneously but selects from
predefined width multipliers rather than discovering continuous
structure.
Our framework lets continuous optimization discover the
architecture from the interaction of task gradients and penalty
pressure, without a predefined search space.
Compute-optimal scaling laws~\citep{kaplan2020scaling,
hoffmann2022training} allocate a fixed budget across model size
and training tokens for a given global architecture.
Our work addresses a complementary question: given a loss target,
what is the minimal-compute \emph{internal} structure that
sustains it?

\section{Structural Regularization}
\label{sec:theory}

We derive the penalty criterion used to suppress channels during training.
Section~\ref{sec:setting} defines the compressible axes.
Section~\ref{sec:axes} shows how additive group-lasso penalties
resolve the gauge freedom between coupled factors
by driving channels toward gauge balance.
Section~\ref{sec:calibration} calibrates the resulting balance by freed compute
to obtain a common utility-per-FMA criterion.

\subsection{Transformer Setting}
\label{sec:setting}

We target the transformer family underlying LLaMA and its descendants:
SwiGLU MLPs~\citep{shazeer2020glu} with RMS normalization.
Because standard RMS normalization includes all channels in its denominator,
suppressing a channel discontinuously rescales the rest;
we introduce a \emph{prescaled} variant that eliminates this discontinuity (Appendix~\ref{app:prewrms}).

Each transformer block contains an attention sub-block and a SwiGLU MLP sub-block.
Each sub-block reads from the residual stream of width~$D$ through an extractor
(selecting $d_{ai}$ or $d_{mi}$ channels)
and writes back through an injector
(routing $d_{ao}$ or $d_{mo}$ channels to the residual stream).
Attention uses query, key, value, and output
projections~$W_q, W_k, W_v, W_o$,%
\footnote{Operators act from the left on column vectors:
$\vq = \mW_q \vx$, stored as $\mQ = \mX \mW_q^\top$.}
operating on~$d_k$ and~$d_v$ dimensions per head across~$H$ heads.
The SwiGLU MLP uses up, gate, and down projections~$W_u, W_g, W_d$,
mapping through filter dimension~$d_f$.

The structural axes of interest are the seven per-block dimensions
($d_{ai}$, $d_k$, $d_v$, $d_{ao}$, $d_{mi}$, $d_f$, $d_{mo}$), each independently adjustable per
block, and two global axes: the residual stream width~$D$ and the classifier input width $d_c$.
We introduce a prescaled RMS normalization (see Appendix~\ref{app:prewrms}) with each extractor
and learnable post-scaling on each injector
to ensure that channels exit smoothly under regularization pressure.
Appendix~\ref{app:axis_inventory} provides the full axis inventory and penalty assignments.

\subsection{Group Lasso and Gauge Balance}
\label{sec:axes}

Consider the value dimension~$d_v$ within an attention head.
The value projection~$W_v$ creates a $d_v$-dimensional representation
and the output projection~$W_o$ maps it back to the residual stream.
Rescaling the $j$-th column of~$W_v$ by~$\alpha$
and the $j$-th column of~$W_o$ by~$1/\alpha$
preserves the product~$W_o W_v$.
Although a penalty on $\|W_v\| \cdot \|W_o\|$ would be gauge invariant,
regularization pressure amplifies any norm imbalance making it unstable.
Appendix~\ref{app:failure_modes} examines this and related pathologies in detail.
This problematic \emph{gauge freedom} appears throughout transformers,
wherever multiplicatively coupled factors share a structural channel.

Group lasso~\citep{yuan2006model} provides the resolution.
Rather than penalizing the product norm, we penalize each factor
separately with a Frobenius-norm group-lasso term that drives all
parameters in a structural slice toward zero simultaneously.
A \emph{structural slice}~$T_j$ collects the parameters removed
when dropping channel~$j$ within one factor.
The regularized loss is
$\mathcal{L}_{\mathrm{reg}}
  = L + \textstyle\sum_{p,j} \lambda_p(j)\, \|T_j^{(p)}\|_F$,
where $\lambda_p(j) > 0$ is the penalty weight for slice~$j$ of
factor~$p$.
On an axis where multiple tensors share a gauge (e.g., the
ensemble $\{W_q, W_k, W_v\}$ on~$d_{ai}$),
their parameters form a single lasso group penalized through a
joint Frobenius norm.

The \emph{utility}~$\omega_j$ of a structural slice is the
first-order task-loss increase from removing it.
At a critical point, utility equals penalty:

\begin{proposition}[Equilibrium identity]
\label{prop:equilibrium}
Suppose the lasso groups form a disjoint partition of the
penalized parameters.
For any structural slice~$T_j$ with $\|T_j\|_F > 0$ and penalty
$\Omega = \lambda\|T_j\|_F$, at a critical point the utility
of~$T_j$ equals its penalty:
$\omega_j = \lambda\|T_j\|_F$.
\end{proposition}

\begin{proof}
At the critical point,
$\mathrm{d}T + \frac{\lambda T}{\| T \|_F} = 0$, so
$\omega = -\mathrm{tr}(T^T\,\mathrm{d}T)
       = \frac{\lambda}{\| T \|_F}\,\mathrm{tr}(T^T\!T)
       = \lambda\|T\|_F$.
\end{proof}

When lasso groups are not disjoint, which occurs whenever a tensor has more than one structural axis,
the identity requires a correction term. See Appendix~\ref{app:corrected_harm} for extended analysis.
The disjoint case captures the essential structure, but the corrected form is used in practice.     

The penalty must act on both sides of each structural axis (i.e., gauge-connected tensor slices)
to ensure stable convergence to a gauge-balanced network
(tested in Appendix~\ref{app:gauge_ablation}).
The additive penalty
$\lambda_A\|\alpha A_j\| + \lambda_B\|\alpha^{-1} B_j\|$
is stable under rescaling;
its minimum over~$\alpha$ occurs when both terms equal $\sqrt{\lambda_A\lambda_B \|A_j\| \|B_j\|}$.
This gauge-balance is monotone in the product $\|A_j\|\cdot\|B_j\|$,
meaning that additive penalties recover the operator component ranking at equilibrium.

\begin{proposition}[Gauge-balanced utility]
\label{prop:balance}
Let two gauge-connected factors have
per-index penalty
$\lambda_A\|A_j\|_F + \lambda_B\|B_j\|_F$
and per-index gauge freedom
($A_j \to \alpha A_j$, $B_j \to B_j/\alpha$ for $\alpha > 0$).
Define the geometric mean
$\mathrm{GM}_j = \sqrt{\lambda_A \lambda_B \|A_j\|_F\,\|B_j\|_F}$.
\begin{enumerate}[label=(\alph*),itemsep=2pt]
\item The penalty has a unique minimum over~$\alpha$:
$\lambda_A\|A_j\|_F = \lambda_B\|B_j\|_F = \mathrm{GM}_j$.
\item At a critical point, gauge balance holds for every index~$j$
 to give the utility $\omega_j = \mathrm{GM}_j$.
\item $\mathrm{GM}_j$ is gauge-invariant and computable from parameters at any training state.
\end{enumerate}
\end{proposition}

\begin{proof}
(a)~Taking $\partial_{\alpha}\left[ \alpha\lambda_A\|A_j\| + \lambda_B\|B_j\|/\alpha \right] = 0$
gives $\alpha = \sqrt{(\lambda_B \|B_j\|)/(\lambda_A\|A_j\|)}$ at the minimum.
Substitution recovers the stated result.
(b)~The gauge direction is loss-invariant, so imbalance contradicts stationarity.
At balance, $\omega_j = \lambda_A\|A_j\| = \mathrm{GM}_j$ (Proposition~\ref{prop:equilibrium}).
(c)~$\sqrt{\lambda_A \lambda_B \|\alpha A_j\|\,\|B_j/\alpha\|} = \sqrt{\lambda_A \lambda_B \|A_j\|\,\|B_j\|}$.
\end{proof}

When lasso groups overlap across axes, a correction is needed
(Appendix~\ref{app:corrected_harm}).
Simultaneous balance across all axes is achieved by gauge Sinkhorn
iterations (Section~\ref{sec:training}).

\subsection{Efficiency-Scaled Calibration}
\label{sec:calibration}

Different structural axes have different element counts and FMA costs per channel.
For example, the prescaled RMS weights~$\gamma$ multiply only one scalar per channel on $d_{ai}$,
whereas the attention ensemble $\{W_q, W_k, W_v\}$ requires $H(2d_k + d_v)$ FMAs per channel.
Without calibration, the penalty exerts uncoordinated pressure across axes.
The following proposition derives the calibration needed to optimize computational efficiency.

\begin{proposition}[Efficiency-scaled calibration]
\label{prop:calibration}
Let $n_A$, $n_B$ be the element counts of the structural slices
comprising gauge-connected factors $A$ and $B$ at position~$j$,
and let $r_j$ be the total FMAs per token for channel~$j$
(including operations such as scaled dot-product attention).
Define the per-element RMS
$\varepsilon_A = \|A_j\|_F / \sqrt{n_A}$ and
$\varepsilon_B = \|B_j\|_F / \sqrt{n_B}$.
Then the unique penalty coefficients under which the efficiency
$\eta_j \equiv \omega_j / r_j = \rho\,\varepsilon_j$ is uniform
across all axes, with
$\varepsilon_A = \varepsilon_B \equiv \varepsilon_j$
at gauge balance, are
\begin{equation}
\label{eq:efficiency_lambda}
\lambda_A = \rho\,r_j/\sqrt{n_A},
\qquad
\lambda_B = \rho\,r_j/\sqrt{n_B}.
\end{equation}
\end{proposition}

\begin{proof}
From the equilibrium identity,
$\eta_j = \omega_j / r_j = \lambda_A \sqrt{n_A}\,\varepsilon_A / r_j$.
Setting $\eta_j = \rho\,\varepsilon_A$ gives
$\lambda_A = \rho\,r_j / \sqrt{n_A}$, and symmetrically
for~$\lambda_B$.
Gauge balance then gives
$\rho\,r_j\,\varepsilon_A = \rho\,r_j\,\varepsilon_B$,
so $\varepsilon_A = \varepsilon_B$.
\end{proof}

Since both $\varepsilon_A$ and $\varepsilon_B$ scales are the same at gauge balance,
step-normalized optimizers (Muon, Adam) apply comparable perturbations to each lasso group.
The global multiplier $\rho > 0$ controls overall compression strength,
which can be fixed or derived adaptively from a loss target.

In summary, Propositions~\ref{prop:equilibrium}--\ref{prop:calibration}
establish that at gauge balance, the additive penalty provides
a calibrated, gauge-invariant measure of each channel's
first-order inference utility per compute.
They are first-order statements at a critical point, which training
only approaches, so the penalty reports an estimate of utility
rather than its exact value.
Section~\ref{sec:training} develops the machinery to apply this
criterion during training, including the coupled updates and gauge
Sinkhorn iterations that drive the network toward the regime where
these statements bind.

\section{Training Procedure}
\label{sec:training}

This section describes how the calibrated penalties from
Section~\ref{sec:theory} are applied during training.
Three mechanisms are essential:
coupled regularization integrates the penalty into each optimizer step;
adaptive~$\rho$ anchors the penalty strength to a loss target;
and trust-drop physically removes channels whose gauge-invariant
utility falls below the learning rate, compacting all affected
tensors to retain dense matrix operations.
Two mechanisms provide robustness: the conservative~$\rho$ solver
restricts compression to steps where the model has already passed the loss target,
and gauge Sinkhorn iterations enforce balance before drop decisions.
Channel exchange, an optional capacity adjustment mechanism, may improve compression quality (Appendix~\ref{app:arch_details}).
Algorithm~\ref{alg:training} shows how these mechanisms are organized in our prototype.

\paragraph{Coupled regularization.}
The penalty gradient must compete with the task gradient through
the same optimizer scaling; otherwise the equilibrium identity
(Proposition~\ref{prop:equilibrium}) does not hold in practice.
If $g$ is the task gradient and
$r = \sum_p \lambda_p T_j^{(p)} / \|T_j^{(p)}\|_F$ is the
unscaled regularization gradient,
the effective gradient
$g_{\mathrm{eff}} = g + \rho\, r$ enters the optimizer as a
single update.
We track momentum on $g$ alone so that both $\rho$ and $r$
respond immediately to structural changes.

\paragraph{Adaptive $\rho$.} \label{sec:adaptive_rho}
A fixed~$\rho$ can under-compress or inhibit training.
Adaptive $\rho$ anchors the penalty strength to a loss
target~$L^*$ by solving the~$\rho$ needed for the post-step loss
to hit the target.
The conservative solver sets
\begin{equation}
\label{eq:conservative_rho}
\rho = \bigl(L^* - L\bigr) / \bigl(-\mathrm{lr}\cdot \langle \mathrm{sign}(m),\, r\rangle\bigr),
\end{equation}
where $m$ is the optimizer momentum.
This gives $\rho > 0$ only when $L < L^*$: compression occurs
only after the model has reached the target, without assuming
the optimizer will improve the loss further.
If the model does improve, the margin $L^* - L$ grows and the
penalty strengthens automatically.
We set $\rho = 0$ when the denominator is not sufficiently
negative (penalty does not oppose descent).
A more aggressive variant includes a descent credit
$\mathrm{lr}\cdot\|m\|_1$ in the numerator,
spending predicted optimizer gains on compression:
\begin{equation}
\label{eq:adaptive_rho}
\rho = \bigl( (L^* - L) + \mathrm{lr}\cdot\|m\|_1 \bigr)
/ \bigl(-\mathrm{lr}\cdot \langle \mathrm{sign}(m),\, r\rangle\bigr).
\end{equation}
This allows the system to \emph{orbit} the loss target,
compressing continuously rather than waiting for realized margin.
The complexity attractor emerges as the structural profile
at which no channel falls below the trust-drop threshold under
the adaptive~$\rho$ that holds the loss at~$L^*$.

\paragraph{Trust-drop.}
Channels driven toward zero by the penalty must be physically
removed to realize the compression as smaller dense tensors.
The gauge-invariant geometric mean
$\sqrt{\varepsilon_A \cdot \varepsilon_B}$ measures a channel's
distance to zero without requiring exact gauge balance.
A channel is dropped when
\begin{equation}
\label{eq:trust_drop}
\sqrt{\varepsilon_A(j) \cdot \varepsilon_B(j)} \;\leq\; \tau,
\qquad \tau = k \cdot \mathrm{lr},
\end{equation}
where $k \geq 1$ is a trust multiplier.
At $k = 1$, only channels within one optimizer step of zero are
dropped.
When a channel is dropped,
corresponding slices are removed from all tensors touching the
axis.
If all channels on an axis are dropped, the owning operator is
removed entirely.

\paragraph{Gauge Sinkhorn.}
During training the network is not at a critical point, so gauge
balance (Proposition~\ref{prop:balance}) does not hold exactly.
Gauge Sinkhorn is a periodic reparameterization that enforces
balance before drop decisions.
For each gauge-connected pair on each axis,
the balance ratio
$c = \sqrt{\lambda_B\|B_j\|_F / (\lambda_A\|A_j\|_F)}$
rescales $A_j \leftarrow c\,A_j$, $B_j \leftarrow B_j/c$,
preserving the network function exactly on bilinear couplings and up to
a scalar gain where one passes through a prescaled-RMS output
(Appendix~\ref{app:prewrms}).
Momentum transforms contravariantly:
$m_A \leftarrow m_A/c$, $m_B \leftarrow c\,m_B$.
Alternating over all axes converges in 2--3 iterations.

\begin{algorithm}[t]
\caption{Training with adaptive compression}
\label{alg:training}
\begin{algorithmic}[1]
\REQUIRE Loss target~$L^*$, trust multiplier~$k$,
  drop interval~$N$, learning rate schedule~$\mathrm{lr}(t)$
\FOR{each training step~$t$}
  \STATE Compute task loss~$L$ and gradient~$g$
  \STATE Compute efficiency-scaling \eqref{eq:efficiency_lambda} from Proposition~\ref{prop:calibration}
         and prepare penalty gradient~$r$
  \STATE Solve for~$\rho$ via~\eqref{eq:conservative_rho} or~\eqref{eq:adaptive_rho}
  \STATE Take combined optimizer step on $g_{\mathrm{eff}} = g + \rho\, r$
  \IF{$t \bmod N = 0$ and $\rho > 0$}
    \STATE Gauge Sinkhorn: rebalance all gauge-connected pairs
    \STATE Trust-drop: remove all channels that satisfy \eqref{eq:trust_drop}
    \STATE Cascade: remove any operators or blocks with no remaining channels
  \ENDIF
\ENDFOR
\RETURN Converged architecture
\end{algorithmic}
\end{algorithm}

Training proceeds in three phases:
a warmup during which $\rho = 0$ and no compression occurs,
a compression phase under adaptive~$\rho$,
and an optional cooldown during which $\rho$ returns to zero and
the architecture is frozen.
Every structural modification produces a dense tensor;
drops shrink the model's memory footprint and forward-pass cost
while momentum tensors undergo compatible surgery so that
optimization continues without restarting.

\section{Experiments}
\label{sec:experiments}

The experiments address five questions:
(1)~Can the method find dramatically smaller architectures while
preserving task performance? (Section~\ref{sec:exp_ff31})
(2)~Do compressed models match or exceed hand-designed baselines?
(Section~\ref{sec:exp_equal})
(3)~Which components of the framework are necessary?
(Section~\ref{sec:ablations})
(4)~Can these architectures be reached without sustained pressure?
(Section~\ref{sec:posthoc})
(5)~Does dense compaction produce real training speedups?
(Figure~\ref{fig:wallclock})

We evaluate GaugeLasso on four tasks:
polynomial long division over a finite field
(deterministic, exact ground truth),
character-level language modeling on PG-19,
masked autoencoding on CelebA, and
a synthetic retrieval task that serves as the ablation testbed.
On language modeling and CelebA, a $2{\times}$-wider,
$2{\times}$-deeper version of the baseline is compressed under penalty.
On language modeling, the compressed model is evaluated against
the baseline at equal FMA.
On CelebA, structural profiles from the initial compression
guide a second iteration with expanded capacity.
All experiments use the Muon optimizer~\citep{jordan2024muon, bernstein2024muon}
with coupled regularization and efficiency-scaled calibration.
Architecture specifications and full hyperparameters appear in
Appendix~\ref{app:exp_details}, per-seed results in
Appendix~\ref{app:full_population}.

\subsection{Polynomial Long Division over FF31}
\label{sec:exp_ff31}

This experiment tests complexity suppression on a deterministic
mathematical task requiring a non-trivial sequence of steps.
A transformer trained on polynomial long division
over~$\mathbb{F}_{31}$ (8~layers, 13.2M~FMA;
Appendix~\ref{app:exp_details})
reaches 100\% accuracy by step~4k without penalty.
Under loss targets $L^*{=}0.01$ and $0.005$,
converged seeds discover architectures at ${\sim}70$k~FMA
($187{\times}$ compression, up to $255{\times}$) with perfect
exact-match accuracy (Table~\ref{tab:ff31}).
The tighter target is more reliable (18/25 vs.\ 3/5 converged)
but requires twice the training budget;
of the seven that do not converge, two never reach a loss low
enough to begin compressing and five compress $146\times$--$191\times$
and then fail to hold the target.
Structural profiles (Appendix~\ref{app:ff31_profiles}) show that
the task never requires more than three attention blocks,
often discarding the first, with a prominent dip at the third MLP block.

\begin{table}[t]
\centering
\caption{Polynomial division over $\mathbb{F}_{31}$
  (5~seeds at $L^*{=}0.01$; 25 at $L^*{=}0.005$).
  A run has \emph{converged} if it compressed below half the initial
  FMA and reached ${\geq}99\%$ exact match on held-out data;
  statistics are over converged seeds only, and
  Appendix~\ref{app:full_population} reports every seed.}
\label{tab:ff31}
\scriptsize
\begin{tabular}{lcccc}
\toprule
Condition & Converged & Exact Match & FMA & Compression \\
\midrule
No penalty & --- & $1.000$ & 13.2M & $1\times$ \\
$L^*=0.01$, 40k steps & 3/5 & $1.000$ & 63k [63k, 90k] & $209\times$ \\
$L^*=0.005$, 70k steps & 18/25 & $1.000$ & 70k [52k, 89k] & $187\times$ \\
\bottomrule
\end{tabular}

\end{table}

\subsection{Language Modeling and Vision}
\label{sec:exp_equal}
Character-level language modeling tests compression where structural
needs vary continuously across depth.
The $L^*{=}\ln 3$ target achieves perplexity \textbf{2.61}
versus the baseline's 2.65 at matched FMA (Table~\ref{tab:text-transformer});
a 24-layer model tested at greater depth compresses $2.2{\times}$ with 2.34 ppl.
Early layer dimensions narrow while deeper layers retain capacity
(Appendix~\ref{app:text_transformer_profiles}).

CelebA masked autoencoding tests compression on a vision task
with qualitatively different structural demands than language
(Table~\ref{tab:celeba-mae}).
Under a stringent loss target ($L^*{=}0.32$),
the compressed model matches the baseline but $d_k$ and~$d_v$
saturate in every layer (Figure~\ref{fig:celeba_loss032}),
a diagnostic that requires pushing the model near its capacity limit.
With $d_k{=}256$, the network sheds excess capacity in early layers
but retains it in later layers, achieving loss~0.213,
well below the baseline (Figure~\ref{fig:celeba_dk256}).
$D$ and~$d_f$ remain near their ceilings,
identifying the next axes for expansion
(Appendix~\ref{app:celeba_profiles}).

\begin{table}[t]
\centering
\caption{Compression results on language modeling and vision
  (mean $\pm$ std, 5~seeds unless noted).}
\label{tab:equal-compute}
\begin{subtable}{\textwidth}
\centering
\caption{Character-level LM on PG-19.
  A $7{\times}$ larger architecture is compressed to the
  baseline's FMA budget.}
\label{tab:text-transformer}
\scriptsize
\begin{tabular}{llcccc}
\toprule
Condition & Architecture & Compression & Loss & Perplexity & FMA \\
\midrule
Baseline & 6L/6H, $D{=}288$ & $1{\times}$ & $0.975 \pm 0.009$ & $2.652 \pm 0.023$ & 7.60M \\
$L^*{=}\ln 3$ & 12L/6H, $D{=}576$ & $7.1{\times}$ & $\mathbf{0.959 \pm 0.008}$ & $\mathbf{2.610 \pm 0.021}$ & 7.56M \\
$L^*{=}\ln 4$ & 12L/6H, $D{=}576$ & $7.5{\times}$ & $1.009 \pm 0.026$ & $2.745 \pm 0.073$ & 7.17M \\
\bottomrule
\end{tabular}

\end{subtable}
\vspace{6pt}
\begin{subtable}{\textwidth}
\centering
\caption{Masked autoencoder on CelebA.
  The last row uses $d_k{=}256$ (4~seeds).}
\label{tab:celeba-mae}
\scriptsize
\begin{tabular}{llcccc}
\toprule
Condition & Architecture & Compression & Loss & FMA \\
\midrule
Baseline & 6L/8H, $D{=}384$ & $1\times$ & $0.236 \pm 0.001$ & 12.1M \\
$L^*{=}0.32$ & 12L/8H, $D{=}768$ & $8.0\times$ & $0.237 \pm 0.001$ & 11.3M \\
$L^*{=}0.40$ & 12L/8H, $D{=}768$ & $9.4\times$ & $0.251 \pm 0.001$ & 9.6M \\
\midrule
$L^*{=}0.24$, $d_k{=}256$ & 12L/8H, $D{=}768$ & $2.9\times$ & $\mathbf{0.213 \pm 0.000}$ & 50.1M \\
\bottomrule
\end{tabular}

\end{subtable}
\end{table}

\subsection{Retrieval and Ablations}
\label{sec:ablations}

\begin{table}[t]
\centering
\caption{Retriever ablation (mean $\pm$ std over 5 seeds;
  sign step over the 4 that completed).
  Baseline uses conservative solver, gauge Sinkhorn, exchange,
  and $k{=}1$; each row deviates in one dimension.}
\label{tab:retriever-ablation}
\scriptsize
\begin{tabular}{@{}l @{\hskip 3pt}l@{\hskip 0.5pt}l@{\hskip 0.5pt}l@{\hskip 0.5pt}l@{\hskip 6pt} llll@{}}
\toprule
\rotatebox{50}{Config} & \rotatebox{50}{Conserv.} & \rotatebox{50}{Sinkhorn} & \rotatebox{50}{Exch.} & \rotatebox{50}{Trust $k$} & \rotatebox{50}{25k Loss} & \rotatebox{50}{30k Loss} & \rotatebox{50}{30k Acc.} & \rotatebox{50}{FMA} \\
\midrule
Baseline & \checkmark & \checkmark & \checkmark & 1 & $1.657 \pm 1.981$ & $0.231 \pm 0.338$ & $0.983 \pm 0.023$ & 203k \\
No Sinkhorn & \checkmark & & \checkmark & 1 & $0.266 \pm 0.491$ & $0.128 \pm 0.255$ & $0.985 \pm 0.030$ & 255k \\
No Exchange & \checkmark & \checkmark & & 1 & $0.036 \pm 0.001$ & $0.001 \pm 0.000$ & $1.000 \pm 0.000$ & 278k \\
Sign Step & & \checkmark & \checkmark & 1 & $6.939 \pm 0.249$ & $5.554 \pm 1.347$ & $0.246 \pm 0.105$ & 101k \\
Trust $k{=}4$ & \checkmark & \checkmark & \checkmark & 4 & $0.048 \pm 0.008$ & $0.001 \pm 0.001$ & $1.000 \pm 0.000$ & 199k \\
Trust $k{=}16$ & \checkmark & \checkmark & \checkmark & 16 & $0.293 \pm 0.501$ & $0.005 \pm 0.007$ & $1.000 \pm 0.001$ & 206k \\
\bottomrule
\end{tabular}

\end{table}

The retrieval task provides low-cost ablations on a simple model
that requires attention and probes the framework's ability to
compress the residual stream dimension~$D$
(Appendix~\ref{app:retrieval_setup}).
Under the baseline configuration,
the framework compresses to one attention block and one MLP
with $D{=}26$--$38$ (Table~\ref{tab:retriever-ablation}).
The ablations validate our baseline settings:
exchange is the most impactful mechanism (37\% more FMA without it),
gauge Sinkhorn accelerates compression,
and the sign-step $\rho$-solver over-compresses catastrophically,
validating the conservative variant.
The trust multiplier is robust across $k \in \{1, 4, 16\}$;
we use $k{=}1$ as the default.
Per-condition profiles are in Appendix~\ref{app:retrieval_setup}.

\subsection{Reaching the Discovered Architectures}
\label{sec:posthoc}
Three controls ask whether these architectures are reachable without
sustained pressure: the same utility metric applied after training,
an external during-training method, and direct training of the
discovered architecture itself.

\paragraph{Post-hoc pruning.}
Applying the same gauge-calibrated utility metric to fully trained
models isolates \emph{when} compression occurs as the only variable
(Appendix~\ref{app:posthoc_details}).
On FF31, the first channel removed already breaks the loss ceiling.
Without penalty pressure, utility distributes so uniformly
that every channel is needed,
while during-training finds $187{\times}$ compression
(Table~\ref{tab:posthoc}).
On character-level LM at matched steps and FMA, post-hoc achieves
perplexity $3.191$ versus~$2.610$, $22\%$ worse.

\paragraph{An external during-training method.}
We reimplemented CoFi~\citep{xia2022structured}, the closest prior
during-training method, inside our framework and targeted the
measured cost of our own compressed model.
Its controller stalls at $1.97{\times}$ compression where $7.09{\times}$
was requested: the coarse mask families saturate, the multipliers are
still rising at the end of the pruning phase, and the discretized
network keeps every sublayer at full head width.
This may reflect the axes available rather than the method itself,
since the architecture we reach narrows key and value dimensions
\emph{inside} heads that stay alive, which its published mask space
cannot express (Appendix~\ref{app:cofi}).

\paragraph{Training the discovered architecture directly.}
Re-initializing a discovered architecture and training it without
penalty asks whether the pressure is needed once the shape is known.
On retrieval the control matches the compressed model, reaching
accuracy $0.9999$ against $0.9829$ at lower FMA.
On character-level LM the control reaches perplexity $2.640$, against
$2.610$ compressed and $2.652$ for the uniform baseline at matched
FMA: comparable quality, and consistent with parity~\citep{liu2019rethinking}, though this arm
alone cannot settle it.
On FF31 the control clearly fails: four of five seeds finish below
$0.01$ exact match and the fifth reaches $0.805$, against $1.000$
for the compressed model.
Sustained pressure is therefore what \emph{discovers} these
architectures; whether it is additionally required to \emph{train}
them is task-dependent, and on exact algorithmic execution it is.

\begin{table}[ht]
\centering
\caption{Reaching the discovered architectures.
  Post-hoc uses the same utility metric, applied after training;
  from-scratch re-initializes the discovered architecture and trains
  it without penalty.
  Seed counts are given because post-hoc FF31 is scored over the
  2 of 5 seeds that learned the task in phase~A.
  Step counts are realized, not configured caps, and differ across
  rows; only the post-hoc chains are budget-matched to their
  during-training comparator.
  Details in Appendix~\ref{app:posthoc_details}.}
\label{tab:posthoc}
\scriptsize
\begin{tabular}{llcccc}
\toprule
Task & Method & Seeds & Training steps & FMA/tok & Quality \\
\midrule
FF31 & During-training ($L^*{=}0.005$) & 18/25 & 70k & 70k & $1.000$ EM \\
FF31 & Post-hoc (A+B+C) & 2/5 & 50k & 13.2M & $1.000$ EM \\
FF31 & From-scratch & 5 & 70k & 71k & $0.161$ EM \\
\midrule
Char LM & During-training ($L^*{=}\ln 3$) & 5 & 47k & 7.6M & $2.610$ ppl \\
Char LM & Post-hoc (A+B+C) & 5 & 47k & 7.6M & $3.191$ ppl \\
Char LM & From-scratch & 5 & 32k & 7.6M & $2.640$ ppl \\
Char LM & Uniform baseline & 5 & 32k & 7.6M & $2.652$ ppl \\
\midrule
Retrieval & During-training & 5 & 30k & 203k & $0.9829$ acc \\
Retrieval & From-scratch & 5 & 30k & 185k & $0.9999$ acc \\
\bottomrule
\end{tabular}

\end{table}

\paragraph{Training acceleration.}
Figure~\ref{fig:wallclock} summarizes wall-clock step time and
loss across all four tasks.
Because structural changes retain dense tensors,
step time falls monotonically as the model compresses,
yielding $1.6$--$5{\times}$ speedups.
During compression, adaptive~$\rho$ makes the loss orbit the target~$L^*$;
with penalties removed during cooldown, loss descends further
at the compressed architecture.
Realized speed depends on both FMA and shape;
with the penalty disabled at matched batch and hardware,
the full-size 12-layer model takes $28.3$s per 100 steps,
the $7.1{\times}$ FMA-compressed architecture $11.9$s ($2.4{\times}$ in time),
and an FMA-matched uniform baseline $6.9$s, with fewer, larger operators.

\begin{figure}[t]
\centering
{\small\textbf{Wall-clock step time and loss during training across all four tasks.}}
\includegraphics[width=\textwidth]{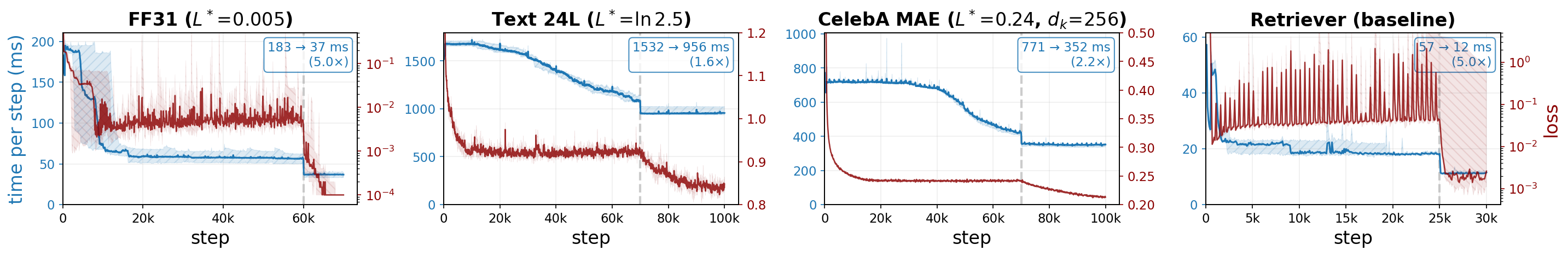}
\caption{Blue (left axis): wall-clock step time.
Red (right axis): task loss.
Grey dashed lines mark cooldown onset.
The text panel uses a 24-layer model
(Appendix~\ref{app:24L_profiles}).
Hatched bands show the range across seeds.}
\label{fig:wallclock}
\end{figure}

\section{Discussion}
\label{sec:discussion}

GaugeLasso compresses dense transformers during training.
Additive group-lasso penalties resolve the gauge pathology that renders naive penalties
unworkable, recovering the product-norm ranking at equilibrium.
Efficiency-scaled calibration puts all channels on a common scale and adaptive
penalty strength tracks the loss target, yielding a principled per-channel measure
of inference utility.
The result is heterogeneous architectures that remain fully dense,
allowing optimized GPU kernels.

\textbf{The complexity attractor.}
Under a fixed loss target, the adaptive~$\rho$ solver and
trust-drop criterion converge to a structural profile of
near-minimal compute sustaining that loss.
Different targets trace different points on a
loss-versus-compute Pareto front, each a distinct architecture
discovered by training.
Shape converges exactly; widths converge in proportion to what they cost.
Across ten runs of one fixed retrieval configuration, every seed keeps
exactly one attention block and one MLP block, and the residual stream
spans $23$--$39$ channels, while the far cheaper MLP interior spans
$1$--$20$.
Compute is dominated by tensors with large axes, so the total varies
much less than the MLP interior does.
On language modeling and vision, seeds likewise agree on aggregate
capacity while differing in per-layer allocation.

\textbf{Representation reorganization.}
The post-hoc comparison (Section~\ref{sec:posthoc}) confirms that
during-training compression discovers representations that are not
subsets of unconstrained training.
The same metric applied post-hoc either fails to compress (FF31)
or reaches substantially worse quality (character-level LM).
These representations emerge from sustained resource pressure that
drives the network to reorganize how it processes information, not
from removing redundancy from a fixed solution.
Discovery and reachability are distinct, however: the architecture
cannot be extracted from a converged model, but once discovered it is
portable on our statistical tasks, which recover comparable quality
from fresh initialization.
Retraining failed on exact algorithmic execution.

\textbf{Iterative architecture design.}
A compression trial reports more than a smaller model.
The structural profile identifies which axes were over-provisioned and
which saturated their budget: axes that compress far were larger than
the task requires, while axes that hold at their ceiling were binding.
Both readings inform the next trial, so successive rounds start closer
to the target and spend less of the run finding it.
On CelebA, $d_k$ saturation in the first iteration guides a second
that allocates capacity where the task needs it, improving loss
from~0.237 to~0.213 (Section~\ref{sec:exp_equal}).

\textbf{Limitations.}
The number of attention heads is fixed throughout training;
extending the penalty to heads is left to future work.
Channel regrowth through nonlinearities is not supported, so the
conservative solver and trust-drop criterion are designed to prevent
premature removal (Appendix~\ref{app:regrowth}).
Discovered architectures narrow more than they shorten, so at these
scales launch and memory latency limit realized speed more than
arithmetic, and FMA reduction overstates wall-clock gain
(Figure~\ref{fig:wallclock}).
This work contributes the mathematics of coherent, FMA-calibrated
utility suppression, tested on architectures small enough to study in
full: every experiment runs below $150$M~FMA and trains from scratch.
That scope makes the structural profiles comparable seed by seed and is
the basis for extending to larger pretrained models; the theory is
architecture-independent.


\section{Conclusion}
The central finding is that sustained structural pressure during training discovers representations more efficient than post-hoc methods can extract from converged models.
GaugeLasso makes this possible.
Dynamically calibrated, gauge-invariant channel utilities drive dense compression without disrupting optimization.
The compressed architectures are not just smaller,
they reflect depth-dependent specialization that fixed-capacity training cannot discover,
and they are portable, retraining to comparable quality from fresh
initialization on our statistical tasks.
Dense compaction accelerates training itself,
and each trial returns a structural profile that informs the next architecture.
As AI adoption grows, inference energy per token is an economic and environmental priority;
this work shows that substantial gains are available not by scaling up,
but by enabling networks to reorganize into more efficient internal structure.

\begin{ack}
This work was supported by ANR-22-CE23-0002 ERIANA,  ANR-22-EXES-0009 MAIA and was granted access to the HPC resources of IDRIS under the allocation 2026-AD011013338 made by GENCI.
The authors declare no competing interests.

Claude Code (Anthropic) was used as a development tool to assist with
code implementation, document organization, and technical exposition.
\end{ack}

\bibliographystyle{plainnat}
\bibliography{references}

\ifchecklist
\fi

\newpage
\appendix

\section{Design Pathologies and Their Resolutions}
\label{app:failure_modes}

Each design decision in the framework resolves a specific failure
mode.
This appendix catalogs these failure modes, progressing from
penalty alternatives through the fundamental instability of
product-norm penalties to calibration and implementation.

\paragraph{$L_2$ and $L_1$ regularization.}
$L_2$ regularization applies pressure proportional to magnitude,
so weights shrink but never reach exactly zero; it cannot drive
channels to extinction.
$L_1$ regularization drives individual weights to zero but
operates on elements, not structural slices: a tensor may become
highly sparse while remaining at full dimension, preventing
compaction into a smaller dense tensor.
Group lasso coordinates suppression across all elements of a
structural slice simultaneously.

\paragraph{The product-norm instability.}
Consider the simplest gauge-connected setting: a rank-1 operator
$P = ba^T$ contributing to $y = BAx$ along one channel, where
$b \in \mathbb{R}^m$ and $a \in \mathbb{R}^n$.
Every unitarily invariant norm of~$P$ reduces to
$\|b\|_2\|a\|_2$ (the sole singular value), so penalizing any
norm of the product is equivalent to penalizing a function of
$\|b\|\cdot\|a\|$.

Let $\Omega = f(\|b\|\cdot\|a\|)$ for any monotonically
non-decreasing~$f$.
The gradients are
\[
\nabla_a \Omega
  = f'(\|b\|\|a\|)\,\|b\|\,\frac{a}{\|a\|},
\qquad
\nabla_b \Omega
  = f'(\|b\|\|a\|)\,\|a\|\,\frac{b}{\|b\|}.
\]
The gradient magnitude on~$a$ is $f'(\cdot)\|b\|$; on~$b$ it is
$f'(\cdot)\|a\|$.
Two pathologies follow, regardless of the choice of~$f$:

\emph{Unstable equilibrium.}\quad
At $\|a\| = \|b\|$, both sides receive equal pressure.
But the equilibrium is unstable: if $\|a\| > \|b\|$ by any
amount, $b$~receives stronger pressure ($f'(\cdot)\|a\|$) and
shrinks faster, while $a$~receives weaker pressure
($f'(\cdot)\|b\|$) and shrinks more slowly.
The imbalance grows with each step.
No choice of~$f$ can make the balanced point attracting.

\emph{Vanishing gradients at zero.}\quad
As the penalty succeeds and $\|b\|\|a\| \to 0$, the instability
drives one side to zero while the other retains mass.
Say $\|a\| \to 0$ with $\|b\|$ large.
The gradient on~$b$ is $f'(\|b\|\|a\|)\cdot\|a\| \to 0$:
the penalty loses its grip on the surviving factor, leaving
dangling weights that prevent tensor compaction.

In contrast, the additive penalty
$\lambda_A\|a\| + \lambda_B\|b\|$ has gradients
$\nabla_a = \lambda_A\, a/\|a\|$ and
$\nabla_b = \lambda_B\, b/\|b\|$: constant magnitude,
independent of the partner factor.
Both sides receive persistent pressure regardless of the other's
state, and gauge balance is a \emph{stable} minimum via AM-GM
(Proposition~\ref{prop:balance}).

\paragraph{The gauge problem.}
A gauge rescaling $\alpha > 0$ changes $\|A_j\|$ and $\|B_j\|$
arbitrarily while preserving the function.
One-sided penalties (penalizing only~$A$) cannot force
coordinated removal: the unpenalized factor~$B$ retains arbitrary
norm, and gauge balance cannot emerge.
The additive symmetric penalty resolves this by applying
persistent pressure on both sides, producing importance estimates
that are gauge-invariant
(Proposition~\ref{prop:balance}c).

\paragraph{Calibration.}
Without calibration, a scalar normalization parameter (one
element) and an attention ensemble ($H(2d_k{+}d_v)$ elements)
receive the same total penalty per channel, creating per-element
pressure asymmetry that overwhelms gauge balance.
Even with per-element pressure equalized, different structural
axes free different amounts of compute per channel: a $d_k$
channel eliminates scaled dot-product attention operations with
no parameter representation, while a $d_f$ channel does not.
Efficiency-scaled calibration
(Proposition~\ref{prop:calibration}) resolves both problems by
scaling the penalty in proportion to total freed compute, which
automatically equalizes per-element RMS at gauge balance.

\paragraph{The coupling problem.}
If the penalty enters as a separate gradient step rather than
inside the optimizer, the penalty and task gradients operate in
different metrics.
Under any step-normalized optimizer (e.g., Muon, normalized
gradient methods), the task gradient receives a normalization
that the decoupled penalty does not.
The equilibrium identity
(Proposition~\ref{prop:equilibrium}) fails: the penalty cannot
overpower the task signal in the optimizer's normalized metric.
Coupled regularization
($g_{\mathrm{eff}} = g + \rho\,r$ entering the optimizer
together) restores the identity.

\paragraph{Density and timing.}
Sparse pruning zeroes channels but keeps tensors at full size;
dense slice removal produces physically smaller tensors at full
hardware utilization.
Applying compression during training rather than post-hoc allows
the network to reorganize representations into fewer dimensions,
discovering architectures that do not exist as subsets of
unconstrained training (Section~\ref{sec:posthoc}).

\medskip
The \emph{penalty} is additive
($\lambda_A\|A\| + \lambda_B\|B\|$), while the \emph{utility estimate}
at gauge balance is
$\mathrm{GM}_j = \sqrt{\lambda_A\lambda_B\|A\|\|B\|}$, a monotonic
function of the product~$\|A\|\cdot\|B\|$.
The additive penalty therefore recovers the product-of-norms ranking at
equilibrium without the vanishing-gradient pathology
(Proposition~\ref{prop:balance}).

\section{Supplementary Theory}
\label{app:supplementary}

\subsection{Corrected Utility for Overlapping Dimensions}
\label{app:corrected_harm}

When a weight tensor has multiple structural axes, penalizing
each axis independently produces overlapping slices.

Consider a weight tensor~$W$ with $A$ adaptive dimensions of
extents $d_1, \ldots, d_A$ and additional \emph{contracted
dimensions}, i.e.\ dimensions shared among every lasso group.
The \textbf{penalty table}
$P \in \mathbb{R}^{d_1 \times \cdots \times d_A}$
collapses the contracted dimensions via Frobenius norm:
\[
P[j_1, \ldots, j_A]
  = \bigl\|W[\,:\,,\; j_1, \ldots, j_A]\bigr\|_F,
\]
where the colon spans all contracted dimensions.
Each entry of~$P$ is the norm of the weight subtensor at a
single combination of adaptive indices, summed over all
dimensions that do not vary across lasso groups.

Write $P_j^{(p)}$ for the $(A{-}1)$-way slice of~$P$ at
index~$j$ along dimension~$p$:
\[
\|P_j^{(p)}\|_F
  = \sqrt{\sum_{j_1,\ldots,j_A:\, j_p = j}
      P[j_1,\ldots,j_A]^2}.
\]
The total penalty is
$\Omega = \sum_{p=1}^{A} \lambda_p
         \sum_{j=1}^{d_p} \|P_j^{(p)}\|_F$.

\paragraph{Multi-dimensional variational condition.}
The gradient of~$\Omega$ with respect to a weight element~$w$
at penalty-table index $(j_1, \ldots, j_A)$
collects contributions from all $A$ dimensions:
\[
\frac{\partial \Omega}{\partial w}
  = w \sum_{p=1}^{A}
    \frac{\lambda_p}{\|P_{j_p}^{(p)}\|_F}.
\]
At the critical point,
$\partial L / \partial w
  = -w \sum_{p}
    \lambda_p / \|P_{j_p}^{(p)}\|_F$.

\paragraph{Corrected utility.}
The first-order utility of removing index~$j$ along
dimension~$p$ is
\begin{align*}
\omega_p[j]
  &= \sum_{q=1}^{A} \lambda_q
     \sum_{j_1,\ldots,j_A:\, j_p = j}
     \frac{P[j_1,\ldots,j_A]^2}{\|P_{j_q}^{(q)}\|_F}.
\end{align*}
Separating the $q = p$ self-term from $q \neq p$ cross-terms:
\[
\omega_p[j]
  = \lambda_p \|P_j^{(p)}\|_F
  + \sum_{q \neq p} \lambda_q
    \sum_{j_1,\ldots,j_A:\, j_p = j}
    \frac{P[j_1,\ldots,j_A]^2}{\|P_{j_q}^{(q)}\|_F}.
\]
The self-term recovers the standard group-lasso utility
(Proposition~\ref{prop:equilibrium}).
The cross-terms are sums of non-negative quantities, so
$\omega_p[j] \geq \lambda_p\|P_j^{(p)}\|_F$: the
single-dimension penalty underestimates true utility when
slices overlap.

\subsection{Regrowth Constraints}
\label{app:regrowth}

Which structural axes admit function-preserving extension?
Linear axes ($d_v$, $d_k$ without RoPE) have no intervening
nonlinearity, so the full orthogonal group $O(n)$ is available:
Procrustes QR can add channels that preserve the existing
function exactly.
Nonlinear axes (those with intervening elementwise activations
such as SwiGLU or PrescaleRMS) admit only permutation matrices,
restricting function-preserving operations to channel reordering.
Table~\ref{tab:regrowth_class} classifies all structural axes.
In this paper, the conservative solver and trust-drop criterion
prevent premature removal, so regrowth is not needed in
practice.

\begin{table}[ht]
\centering
\caption{Axis classification by intervening transformation.
Linear axes admit full basis rotation, which allows
Procrustes QR growth.
Nonlinear axes admit only channel permutation.}
\label{tab:regrowth_class}
\small
\begin{tabular}{@{}llll@{}}
\toprule
Axis & Type & Commutant & Available operations \\
\midrule
$d_v$ & Linear & $O(n)$ & Procrustes QR \\
$d_k$ (no RoPE) & Linear & $O(n)$ & Procrustes QR \\
\midrule
$d_k$ (RoPE) & Nonlinear & Block $SO(2)$
  & Pair-wise rotation (deferred) \\
$d_f$ & Nonlinear & $S_n$ & Permutation only \\
$d_{ai}$, $d_{mi}$ & Nonlinear & $S_n$ & Permutation only \\
$d_{ao}$, $d_{mo}$ & Nonlinear & $S_n$ & Permutation only \\
$D$ (residual) & Nonlinear & $S_n$ & Permutation only \\
\bottomrule
\end{tabular}
\end{table}

\subsection{Cross-Count Calibration}
\label{app:cross_count}

The efficiency-scaled calibration of
Section~\ref{sec:calibration} chooses penalty weights
$\hat\lambda_A = r_j/\sqrt{n_A}$,
$\hat\lambda_B = r_j/\sqrt{n_B}$ to normalize utility by FMAs.
The RMS equalization
$\varepsilon_A = \varepsilon_B$ at gauge balance
(Proposition~\ref{prop:calibration}) depends only on the
\emph{ratio} of penalty weights, not on their absolute scale:

\begin{remark}[RMS-equalized calibration]
\label{rem:cross_count}
At gauge balance
($\lambda_A\|A_j\|_F = \lambda_B\|B_j\|_F$),
the per-element RMS values satisfy
$\varepsilon_A = \varepsilon_B$ if and only if
$\lambda_A / \lambda_B = \sqrt{n_B / n_A}$.
\end{remark}

Any shared rescaling
$\hat\lambda_A = c\sqrt{n_B}$,
$\hat\lambda_B = c\sqrt{n_A}$
preserves this ratio, and therefore RMS equalization, while
normalizing channel utility against a resource of the
practitioner's choice.
Setting $c = 1$ gives the simplest geometry-only calibration;
setting $c = r_j/\sqrt{n_A n_B}$ recovers the efficiency-scaled
weights of Section~\ref{sec:calibration}, which normalize by FMAs.
The same degree of freedom could normalize against memory
footprint, latency, or any other per-channel resource cost.

\section{Architecture Details}
\label{app:arch_details}

\subsection{Axis Inventory}
\label{app:axis_inventory}

Table~\ref{tab:axes} provides the complete structural axis
inventory.
The following paragraphs trace the penalty structure in
forward order through the transformer block.

\begin{table}[ht]
\centering
\caption{Structural axes (6 per block, plus global~$D$).
All axes use the geometric mean utility estimate
(Proposition~\ref{prop:balance}).
Each block has four bridge paths (attention/MLP $\times$
input/output), each independently calibrated; the paper
writes $d_{ai}$/$d_{mi}$ and~$d_{ao}$/$d_{mo}$ when the argument is
per-path-independent.}
\label{tab:axes}
\small
\begin{tabular}{@{}lllc@{}}
\toprule
\textbf{Axis} & \textbf{Controls}
  & \textbf{Gauge-connected factors (A;\;B)} & \textbf{GM} \\
\midrule
\multicolumn{4}{@{}l}{\textit{Per-block: attention}} \\
$d_k$ & Key/query channels
  & $W_q[:,:,j]$;\; $W_k[:,:,j]$ & Yes \\
$d_v$ & Value channels
  & $W_v[:,:,j]$;\; $W_o[:,j,:]$ & Yes \\
\midrule
\multicolumn{4}{@{}l}{\textit{Per-block: MLP}} \\
$d_f$ & Filter width
  & $W_u[:,j]$;\; $W_d[j,:]$\,$^*$ & Yes \\
\midrule
\multicolumn{4}{@{}l}{\textit{Per-block: residual interface}} \\
$d_{ai}$ & Attn input
  & $\gamma_j$;\; $\{W_q,W_k,W_v\}$$^\dagger$ & Yes \\
$d_{ao}$ & Attn output
  & $W_o[:,:,j]$;\; $\delta_j$ & Yes \\
$d_{mi}$ & MLP input
  & $\gamma_j$;\; $\{W_u,W_g\}$$^\dagger$ & Yes \\
$d_{mo}$ & MLP output
  & $W_d[:,j]$;\; $\delta_j$ & Yes \\
\midrule
\multicolumn{4}{@{}l}{\textit{Global}} \\
$D$ & Residual width
  & Injector-side scales$^\ddagger$;\;
    extractor-side scales & Yes \\
\bottomrule
\end{tabular}

\smallskip\noindent
{\footnotesize
$^*$\,$W_g[:,j]$ is gauge-disconnected on~$d_f$
(Section~\ref{sec:axes}).\quad
$^\dagger$\,Joint Frobenius within each ensemble.\quad
$^\ddagger$\,Side A: $\{\delta_\mathrm{attn},
\delta_\mathrm{mlp}\}$ and embedding scale across all blocks.
Side B: $\{\gamma_\mathrm{attn}, \gamma_\mathrm{mlp}\}$
across all blocks.}
\end{table}

\paragraph{Residual stream ($D$, $d_{ai}$, $d_{ao}$, $d_{mi}$, $d_{mo}$).}
The global axis~$D$ controls which residual sites exist.
PostScale~$\delta$ across all blocks forms one ensemble factor;
prescaled RMS~$\gamma$ across all blocks forms the other.
Input axes~$d_{ai}$/$d_{mi}$ pair $\gamma_j$ against the
sub-block's consumption operators.
Output axes~$d_{ao}$/$d_{mo}$ pair the operator weight
($W_o$ or~$W_d$) against PostScale~$\delta_j$.

\paragraph{Attention ($d_k$, $d_v$).}
$W_q$ and~$W_k$ are gauge-connected on~$d_k$: rescaling
$W_q[:,:,j] \to \alpha W_q[:,:,j]$,
$W_k[:,:,j] \to W_k[:,:,j]/\alpha$
preserves $Q^T\!K$.
Each factor is penalized independently (sum of norms, not joint
Frobenius), since either vanishing eliminates the channel.
$W_v$ and~$W_o$ are gauge-connected on~$d_v$.
With RoPE, adjacent channel pairs form the structural unit
on~$d_k$.

\paragraph{SwiGLU MLP ($d_f$).}
$W_u$ and~$W_d$ are gauge-connected on~$d_f$.
$W_g$ is gauge-disconnected: zeroing $W_u[:,j]$ kills the
output regardless of~$W_g[:,j]$.
$W_g$ is not penalized on~$d_f$ but joins the~$d_{mi}$
ensemble factor.

\subsection{Freed Resources}
\label{app:freed}

Table~\ref{tab:freed} lists the per-index freed resources for
each structural axis.
The \emph{penalized} column names the gauge-connected factors
(side~A;\;side~B) whose norms enter the penalty;
$n_A$ and~$n_B$ are the element counts of the structural slices;
the efficiency-scaled calibration
(Section~\ref{sec:calibration}) sets
$\hat\lambda_A = r_j/\sqrt{n_A}$,
$\hat\lambda_B = r_j/\sqrt{n_B}$.
The \emph{coupled} column names gauge-disconnected tensors that
share the axis and are freed when an index is dropped but do not
participate in the penalty.
The $r_{\mathrm{compute}}$ column counts per-token multiply-add
savings from all eliminated operations.
Terms involving~$S$ (sequence length) require a target length to
evaluate; for long sequences these dominate.

\begin{table}[ht]
\centering
\caption{Freed resources per dropped index (per block, per token).}
\label{tab:freed}
\small
\begin{tabular}{@{}llllll@{}}
\toprule
\textbf{Axis} & \textbf{Penalized (A;\;B)}
  & $n_A$ & $n_B$
  & \textbf{Coupled}
  & $r_{\mathrm{compute}}$ \\
\midrule
\multicolumn{6}{@{}l}{\textbf{Attention}} \\
$d_{ai}$ & $\gamma$;\; $\{W_q,W_k,W_v\}$
  & $1$ & $H(2d_k{+}d_v)$
  &
  & $1 + H(2d_k{+}d_v)$ \\
$d_k$ & $W_q$;\; $W_k$
  & $Hd_{ai}$ & $Hd_{ai}$
  &
  & $H(2d_{ai}{+}S)$ \\
$d_v$ & $W_v$;\; $W_o$
  & $Hd_{ai}$ & $Hd_{ao}$
  &
  & $H(d_{ai}{+}d_{ao}{+}S)$ \\
$d_{ao}$ & $W_o$;\; $\delta$
  & $Hd_v$ & $1$
  &
  & $Hd_v + 1$ \\
$H^\dagger$ & $W_v$;\; $W_o$
  & $d_{ai} d_v$ & $d_v d_{ao}$
  & $W_q,\,W_k$
  & $r_H\!$ \\
\midrule
\multicolumn{6}{@{}l}{\textbf{SwiGLU MLP}} \\
$d_{mi}$ & $\gamma$;\; $\{W_u,W_g\}$
  & $1$ & $2d_f$
  &
  & $1 + 2d_f$ \\
$d_f$ & $W_u$;\; $W_d$
  & $d_{mi}$ & $d_{mo}$
  & $W_g$
  & $2d_{mi}{+}d_{mo}$ \\
$d_{mo}$ & $W_d$;\; $\delta$
  & $d_f$ & $1$
  &
  & $d_f + 1$ \\
\bottomrule
\end{tabular}

\smallskip\noindent
{\footnotesize\flushleft
$^\dagger$\,Fixed in this work; included for completeness.
$r_H = d_{ai}(2d_k{+}d_v) + S(d_k{+}d_v) + d_v d_{ao}$.
With RoPE, $d_k$ channels come in pairs.
The global axis~$D$ is omitted: its freed resources are
architecture-dependent sums over blocks.}
\end{table}

\subsection{Prescaled RMS Normalization}
\label{app:prewrms}

Prescaled RMS normalization (Section~\ref{sec:setting}) serves
three roles in the penalty structure:
(1)~\textbf{Penalty attachment:} $\gamma$ provides the
gauge-connected factor for~$d_{ai}$ and~$d_{mi}$, and across all blocks forms
the extractor-side ensemble for~$D$.
(2)~\textbf{Gauge mediation:} prescaling by~$\gamma$ before the
denominator restores the residual stream gauge that standard RMS
normalization breaks.
(3)~\textbf{Seamless exit:} as $\gamma_j \to 0$, the
$\|\gamma\|^2$ denominator ensures surviving channels maintain
their normalization.

\paragraph{Scope of the gauge-Sinkhorn invariance.}
Role~(2) is also where the rescaling of
Section~\ref{sec:training} stops being exact.
On a purely bilinear coupling, $A_j \leftarrow c A_j$,
$B_j \leftarrow B_j/c$ cancels identically.
Where the coupling passes through a prescaled-RMS output --- the input
bridges, the final bridge, and the residual stream~$D$, whose consumers
are themselves bridge gammas --- the normalizer depends on $\gamma$
through both the numerator and the sum of squares.
A \emph{global} rescale therefore cancels and is exactly undone by the
consumers, which is why anchoring the residual stream is exact, while a
\emph{per-channel} rescale leaves one data-dependent scalar gain on that
block's output.
The approximation is not observed to matter in practice, but the
invariance should be read as exact on the bilinear couplings and
first-order elsewhere.

\subsection{Procrustes Growth}
\label{app:procrustes}

For gauge-connected pairs with no intervening nonlinearity
($d_v$, $d_k$ without RoPE), the full gauge freedom includes the
orthogonal group~$O(d)$, not just scalar rescalings.
Procrustes QR growth exploits this freedom to grow an axis
from $d$ to $d{+}k$ dimensions.
A random $(d{+}k) \times d$ orthonormal matrix~$\hat{Q}$ is
Procrustes-aligned so that existing channels are minimally
perturbed (scaled by entries of
$\Sigma_{\hat{Q}} \approx 1 - O(k/d)$) while new channels have
$O(1)$ parameter norms and zero net output contribution.
Training immediately differentiates the new channels.
Any function-preserving reparameterization $A' = MA$
requires momentum to transform as $m_A' = M^{-T} m_A$
to preserve the optimizer's gradient basis.
Gauge Sinkhorn is the scalar case where $M = cI$ and $m_A' = m_A/c$).
Procrustes growth is the isometric case,
for which weight and momentum transform identically.

\section{Experimental Details}
\label{app:exp_details}

All experiments were run on NVIDIA H100 GPUs (80\,GB).
Each run uses a single GPU.
The total compute budget across all reported experiments
(including ablations, post-hoc comparisons, and the 24-layer
extension) is approximately 500~GPU-hours on H100; step times
decrease as models compress, so the effective cost is
substantially less than the initial per-step rate would suggest.
Table~\ref{tab:exp_setup} summarizes the architecture and
training configuration for each task.

\paragraph{Datasets.}
PG-19~\citep{rae2020compressive} is used under the Apache~2.0
license.
CelebA~\citep{liu2015faceattributes} is used for non-commercial
research under its original license.
The FF31 polynomial division and retrieval tasks use synthetic
data generated at training time.

\paragraph{Structural profiles.}
For each compression profile figure, the first seven panels plot axis
extent, the number of active channels, against transformer depth for the
seven adaptive axis types
($d_{ai}$, $d_k$, $d_v$, $d_{ao}$, $d_{mi}$, $d_f$, $d_{mo}$).
Line color encodes training step, and the hatched band gives the full
range across seeds at convergence unless a caption says otherwise.
The eighth panel runs against training step instead, carrying $D$ and
the classifier input width~$d_c$ on the left axis and total FMA on the
right.
Extents use a $\log_2$ scale to resolve both large initial and small
converged values.

\begin{table}[ht]
\centering
\caption{
Experimental configurations.
Cells with several entries follow the order of the $L^*$ row.
Two stopping rules are in play.
Conditions with an FMA target (Char LM 12L, CelebA $d_k{=}96$) end
one \emph{Saturation} phase after first meeting both that target
and the loss target, annealing the learning rate over the phase, so
their \emph{step cap} never binds.
The rest set no FMA target: they run to the cap and anneal over the
\emph{Cooldown} instead.
\emph{Realized steps} is the median over seeds of what actually ran.
\emph{Train FMA} is the per-token inference FMA summed over
realized steps, in units of $10^{11}$: a proxy for training cost
that is comparable within a column but not across columns, since
tokens per step differ by task.
On the same scale the baselines cost 2.4 (Char LM) and 7.4
(CelebA).}
\label{tab:exp_setup}
\small
\begin{tabular}{@{}llllll@{}}
\toprule
& \textbf{FF31} & \textbf{Char LM 12L} & \textbf{Char LM 24L} & \textbf{CelebA MAE} & \textbf{Retrieval} \\
\midrule
Task & Poly.\ division & Next-char & Next-char & Masked AE & Mod-arith.\ retr.\ \\
Baseline arch.\ & --- & 6L/6H & --- & 6L/8H & --- \\
Baseline $D$ & --- & 288 & --- & 384 & --- \\
Baseline FMA & --- & 7.6M & --- & 12.1M & --- \\
Baseline steps & --- & 31.8k & --- & 61.2k & --- \\
Adaptive arch.\ & 8L/8H & 12L/6H & 24L/8H & 12L/8H & 4L/4H \\
Adaptive $D$    & 192 & 576 & 480 & 768 & 128 \\                                       
Adaptive FMA & 13.2M & 53.7M & 113M & 90.2M (145.2M$^\dagger$) & 1.9M \\
$d_k$ / $d_v$ & 48 / 48 & 96 / 96 & 90 / 90 & 96 / 96 (256$^\dagger$) & 32 / 32 \\
$L^*$ & 0.01, 0.005 & $\ln 3$, $\ln 4$ & $\ln 2.5$ & 0.32, 0.40 (0.24$^\dagger$) & 0.05 \\
LR          & 1e-3 & 1e-3 & 5e-4 & 1e-3 (5e-4$^\dagger$) & 1e-3 \\
LR$_{\min}$ & 1e-4 & 1e-4 & 1e-4 & 1e-4 & 1e-4 \\
Weight decay & 0 & 0 & 0.01 & 0 (0.01$^\dagger$) & 0 \\
$\beta_{\mathrm{loss}}$ & 0 & 0 & 0.9 & 0 (0.9$^\dagger$) & 0 \\
Batch size & 128 & 32$\times$3 & 48$\times$2 & 32$\times$3 & 32 \\
Warmup & 500 & 500 & 500 & 500 & 500 \\
Saturation & --- & 30k & --- & 60k (---$^\dagger$) & --- \\
Cooldown & 10k, 10k & 0  & 30k & 0 (30k$^\dagger$) & 5k \\
Step cap & 40k, 70k & 1M & 100k & 100k (100k$^\dagger$) & 30k \\
Realized steps & 40k, 70k & 47.7k, 31.9k & 100k & 63.4k, 61.7k & 30k \\
Train FMA & 0.50, 0.52 & 6.0, 2.9 & 76 & 9.0, 7.0 (94$^\dagger$) & 0.11 \\
$\rho$-solver & Conservative & Conservative & Conservative & Conservative & Conservative \\
Sinkhorn & Yes & Yes & Yes & Yes & Yes \\
Exchange & No & No & No & No & Yes \\
Seeds & 5, 25 & 5 & 5 & 5 (4$^\dagger$) & 5 \\
\bottomrule
\multicolumn{6}{@{}l}{\footnotesize $^\dagger$CelebA $d_k{=}256$ condition (Section~\ref{sec:exp_equal}).}\\
\end{tabular}
\end{table}

\subsection{FF31 Polynomial Division}
\label{app:ff31_setup}
\label{app:ff31_profiles}

The task requires polynomial long division of cubics by linears
over~$\mathbb{Z}/31\mathbb{Z}$.
Each example requires three rounds of divide, multiply, subtract,
and accumulate; the model produces all intermediate steps as a
fixed-length sequence (283~tokens, vocab~44, 254~supervised
positions).
Ten percent of training examples are replaced with arithmetic
drill sequences to maintain operation accuracy during structural
compression.
The architecture uses $d_k{=}d_v{=}48$ and
$d_{\mathrm{mlp}}{=}2048$.
Without penalty, the model reaches 100\% token accuracy and
exact-match by step~4k.

Figure~\ref{fig:ff31_profiles} shows the structural profiles
under both loss targets.
Across seeds, the converged architectures retain at most three
attention blocks (often discarding block~0) and show a consistent
dip at MLP block~2, suggesting this position contributes less
to the polynomial division task.
Training details are in Table~\ref{tab:exp_setup}.

\begin{figure}[p]
\centering
\begin{subfigure}{\textwidth}
\caption{FF31, $L^*{=}0.01$ (30k compression + 10k cooldown).}
\label{fig:ff31_A}
\centering
\includegraphics[width=\textwidth]{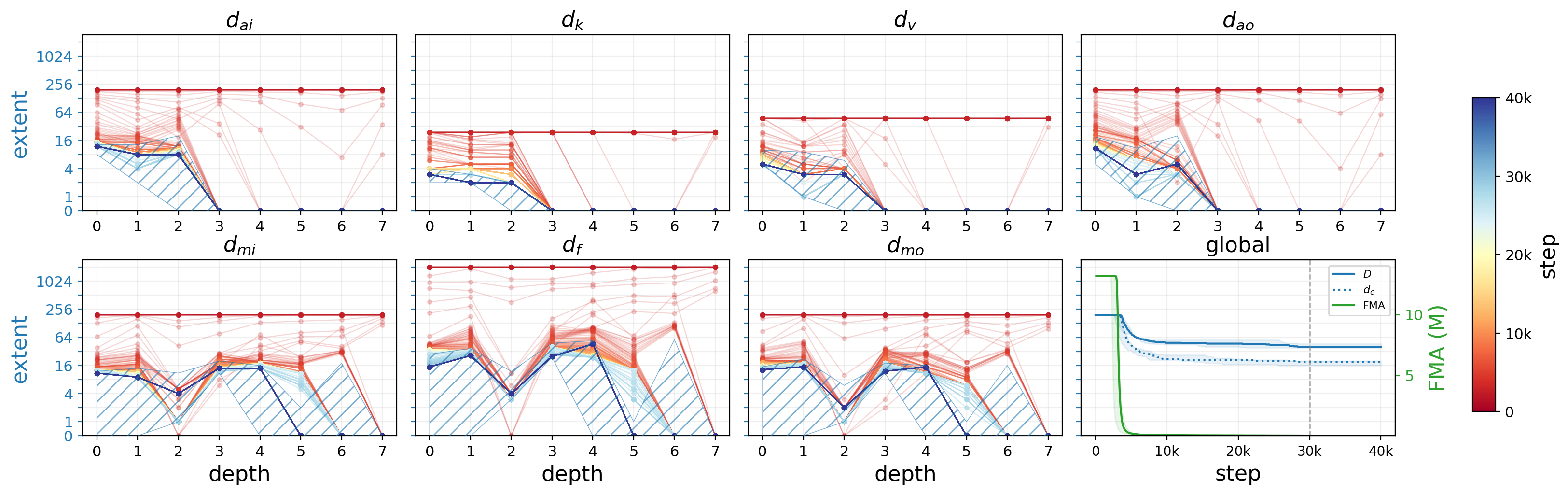}
\end{subfigure}
\begin{subfigure}{\textwidth}
\caption{FF31, $L^*{=}0.005$ (60k compression + 10k cooldown).}
\label{fig:ff31_B}
\centering
\includegraphics[width=\textwidth]{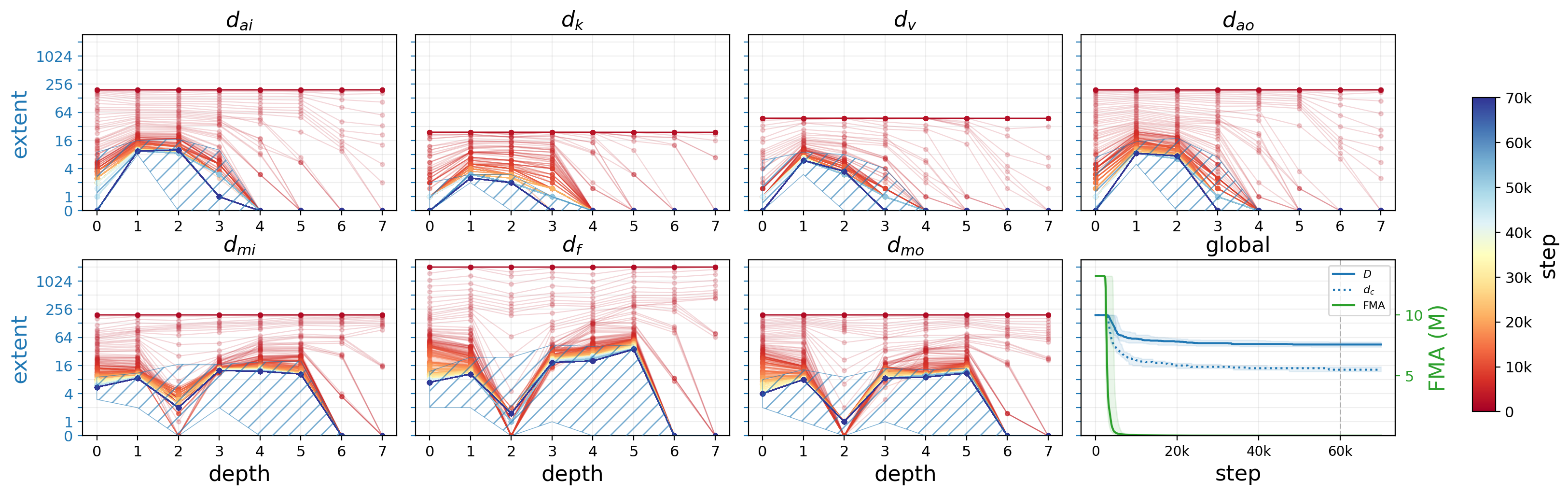}
\end{subfigure}
\caption{FF31 structural profiles.
\textbf{(a)}~Three of five seeds converge to ${\sim}63$k~FMA
with perfect accuracy; the wider loss margin permits one seed
to over-compress.
\textbf{(b)}~18 of 25 seeds converge to ${\sim}70$k~FMA;
the tighter target produces more stable compression
(hatched band: middle 78\% of converged seeds).}
\label{fig:ff31_profiles}
\end{figure}

\subsection{Character-Level Language Modeling}
\label{app:lm_setup}
\label{app:text_transformer_profiles}
\label{app:24L_profiles}

Next-character prediction on PG-19 (byte-level, vocab size~256),
causal attention with RoPE, sequence length~512,
batch size~$32{\times}3$ (gradient accumulation),
Muon optimizer with LR~$10^{-3}$ cosine-decayed to~$10^{-4}$,
500-step warmup.
The baseline (6L/6H, $D{=}288$, $d_k{=}d_v{=}48$,
$d_{\mathrm{mlp}}{=}768$; 7.6M~FMA) trains without penalty.
The adaptive architecture (12L/6H, $D{=}576$,
$d_k{=}d_v{=}96$, $d_{\mathrm{mlp}}{=}1536$; 53.7M~FMA)
compresses until $\texttt{total\_fma} \leq 7.6\text{M}$,
at which point the architecture freezes and training continues
at fixed structure.
Every run then trains for a further 30k~steps, annealing
to~$10^{-4}$; realized lengths are in Table~\ref{tab:exp_setup}.
Matching FMA matches inference compute, not training cost: reaching
its architecture costs the $L^*{=}\ln 3$ condition about
$2.5{\times}$ the baseline's training compute.

To test the framework at greater depth, we also train a 24-layer
model ($D{=}480$, 8~heads, $d_k{=}d_v{=}90$,
$d_{\mathrm{mlp}}{=}1920$; 113M~FMA) with $L^*{=}\ln 2.5$ for
70k~compression steps followed by 30k~cooldown.
The framework compresses to mean $51.9$M~FMA ($2.2{\times}$) at
$2.335 \pm 0.025$~ppl over five seeds.
The compression phase was limited to 70k~steps and further compression
is expected with a longer budget.
Note the more conservative training settings used for the 24L
experiments (Table~\ref{tab:exp_setup}).

Figure~\ref{fig:text_transformer_profiles} shows the structural
profiles for all three conditions.
Tighter loss targets produce smoother, more consistent profiles.
In the 24L model, whole sublayers are removed, but sparsely:
across the five seeds, 0--2 of the 48 attention and MLP
sublayers are eliminated outright (two MLP blocks in one seed, one
attention and one MLP block in another, none in the remaining three).
The FMA that vanishes from the profiles is therefore mostly
vestigial width rather than removal.
What is consistent across seeds sits at the front of the network:
in every converged seed the first attention block is reduced to
$d_k \in [6, 17]$ of~45 and $d_v \in [1, 4]$ of~90, and the first
MLP to $d_f \in [68, 191]$ of~1920, while the second layer still
carries hundreds of residual-stream channels.
In all conditions, the initial attention and MLP blocks are
dramatically narrowed (visually understated due to the log scale of the extent axis).
Table~\ref{tab:per_layer} reports the converged extents of seed~42
directly, on every axis of every layer.

\begin{figure}[p]
\centering
\begin{subfigure}{\textwidth}
\caption{12-layer, $L^*{=}\ln 4$ ($7.5{\times}$ compression, ppl~2.75).}
\label{fig:text_transformer_ppl4}
\centering
\includegraphics[width=\textwidth]{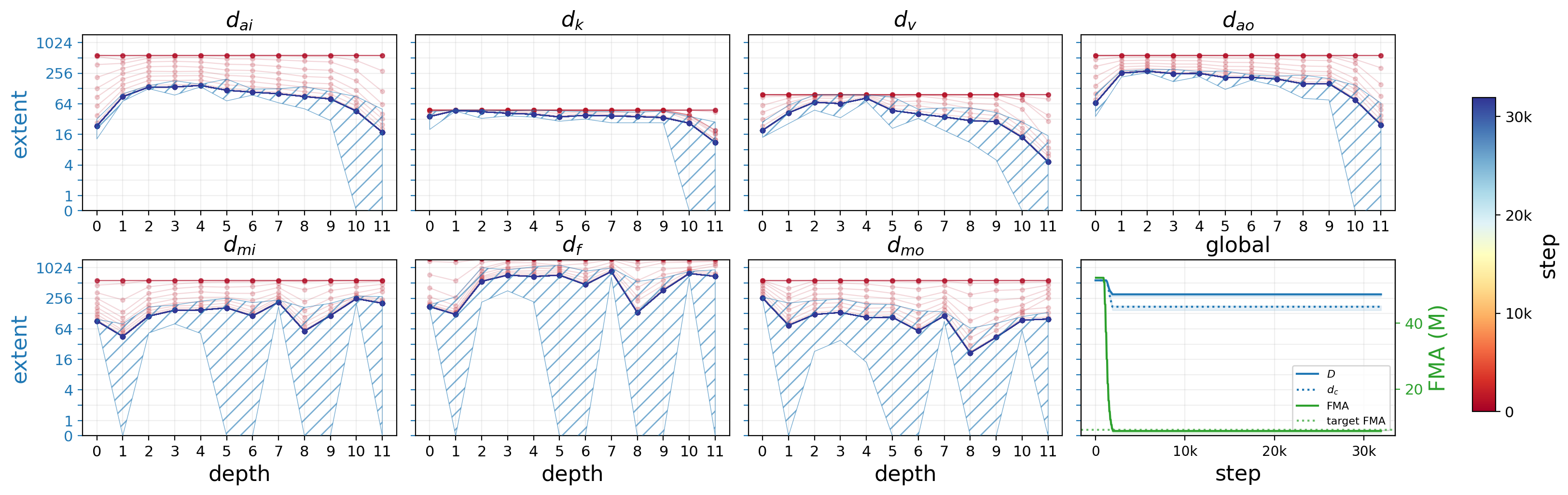}
\end{subfigure}
\begin{subfigure}{\textwidth}
\caption{12-layer, $L^*{=}\ln 3$ ($7.1{\times}$ compression, ppl~2.61).}
\label{fig:text_transformer_ppl3}
\centering
\includegraphics[width=\textwidth]{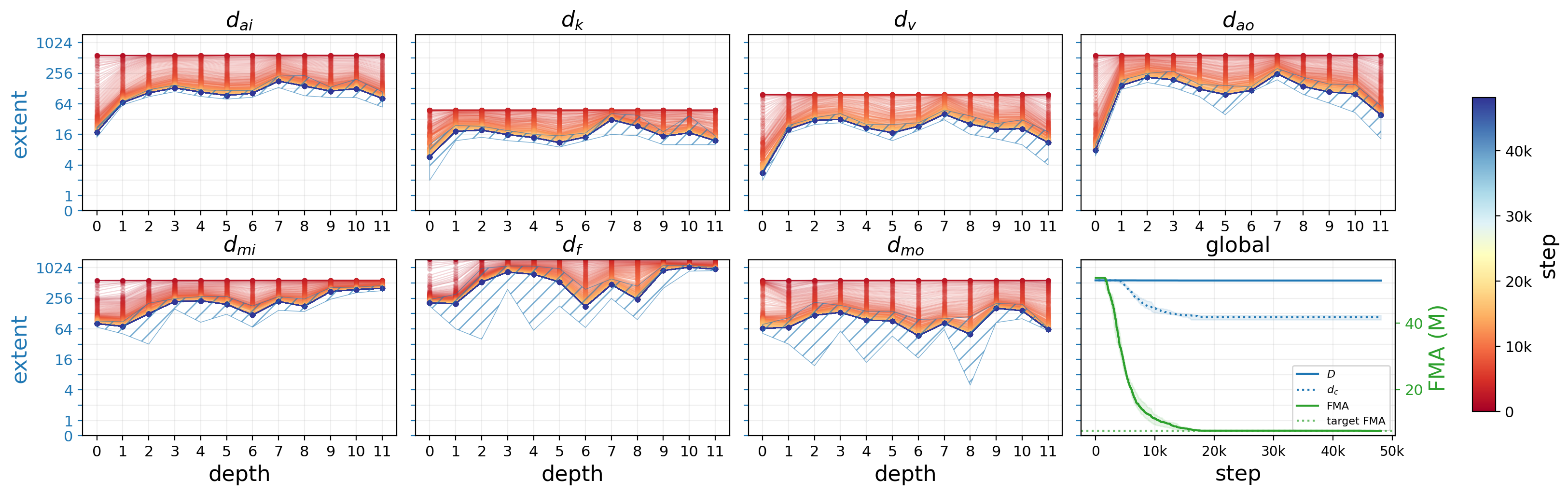}
\end{subfigure}
\begin{subfigure}{\textwidth}
\caption{24-layer, $L^*{=}\ln 2.5$ ($2.2{\times}$ compression, ppl~2.34).
  Grey dashed line marks the start of cooldown at step~70k.}
\label{fig:tt24L}
\centering
\includegraphics[width=\textwidth]{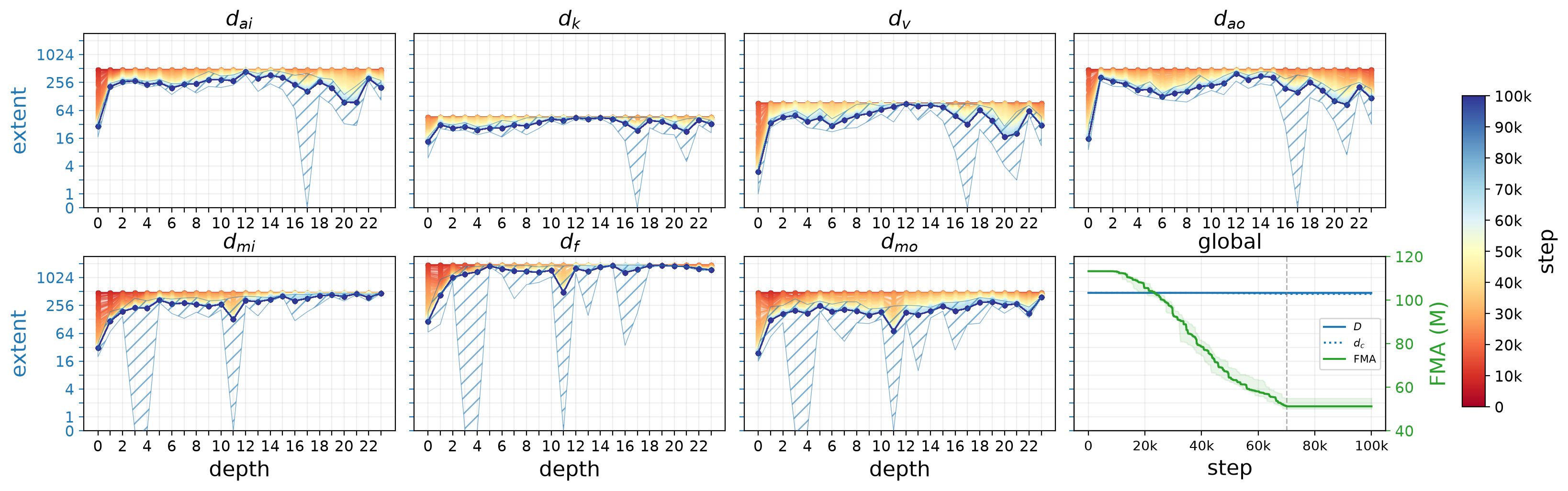}
\end{subfigure}
\caption{Character-level LM structural profiles.
\textbf{(a)}~The loose target compresses early, at the cost of
representational quality.
\textbf{(b)}~The tighter target beats the baseline, and deeper layers
retain more capacity.
\textbf{(c)}~At 24~layers the model narrows channel dimensions across
depth while retaining every layer.}
\label{fig:text_transformer_profiles}
\end{figure}

\begin{table}[htbp]
\centering
\caption{Per-layer structure of the discovered architectures, seed~42;
  seed-to-seed spread is in Figure~\ref{fig:text_transformer_profiles}.
  Each entry is the number of surviving channels on that axis;
  \emph{---} marks an operator removed entirely.
  Ceilings are \perlayerCeilA\ and \perlayerCeilB\ respectively, with
  $d_k$ halved because RoPE pairs share a frequency (Table~\ref{tab:axes}).}
\label{tab:per_layer}
\scriptsize
\begin{tabular}{@{}r@{\hskip 5pt} rrrrr@{\hskip 3pt}rrr@{\hskip 9pt} rrrrr@{\hskip 3pt}rrr@{}}
\toprule
& \multicolumn{8}{c}{\emph{12L, $L^*{=}\ln 3$}} & \multicolumn{8}{c}{\emph{24L, $L^*{=}\ln 2.5$}} \\
\cmidrule(lr){2-9}\cmidrule(lr){10-17}
& \multicolumn{5}{c}{Attention} & \multicolumn{3}{c}{MLP} & \multicolumn{5}{c}{Attention} & \multicolumn{3}{c}{MLP} \\
\cmidrule(lr){2-6}\cmidrule(lr){7-9}\cmidrule(lr){10-14}\cmidrule(lr){15-17}
Layer & $d_{ai}$ & $H$ & $d_k$ & $d_v$ & $d_{ao}$ & $d_{mi}$ & $d_f$ & $d_{mo}$ & $d_{ai}$ & $H$ & $d_k$ & $d_v$ & $d_{ao}$ & $d_{mi}$ & $d_f$ & $d_{mo}$ \\
\midrule
0 & 19 & 6 & 10 & 3 & 8 & 82 & 212 & 78 & 30 & 8 & 15 & 4 & 15 & 31 & 120 & 27 \\
1 & 68 & 6 & 24 & 24 & 186 & 78 & 250 & 63 & 187 & 8 & 26 & 27 & 276 & 70 & 126 & 61 \\
2 & 121 & 6 & 21 & 31 & 211 & 204 & 900 & 187 & 257 & 8 & 23 & 40 & 296 & 217 & 1201 & 256 \\
3 & 111 & 6 & 15 & 28 & 134 & 157 & 384 & 58 & 262 & 8 & 27 & 43 & 211 & --- & --- & --- \\
4 & 118 & 6 & 15 & 25 & 148 & 270 & 894 & 105 & 245 & 8 & 28 & 39 & 202 & 270 & 1672 & 216 \\
5 & 79 & 6 & 10 & 15 & 76 & 172 & 364 & 78 & 258 & 8 & 26 & 42 & 161 & 308 & 1829 & 237 \\
6 & 117 & 6 & 15 & 25 & 113 & 134 & 226 & 69 & 164 & 8 & 20 & 30 & 133 & 308 & 1792 & 215 \\
7 & 172 & 6 & 34 & 41 & 222 & 257 & 613 & 89 & 244 & 8 & 36 & 29 & 117 & 348 & 1818 & 246 \\
8 & 147 & 6 & 30 & 26 & 110 & 148 & 129 & 5 & 210 & 8 & 19 & 36 & 144 & 355 & 1800 & 283 \\
9 & 85 & 6 & 10 & 13 & 66 & 404 & 1141 & 164 & 211 & 8 & 33 & 41 & 152 & 142 & 804 & 104 \\
10 & 113 & 6 & 10 & 23 & 102 & 329 & 873 & 101 & 352 & 8 & 44 & 79 & 255 & 218 & 1193 & 149 \\
11 & 55 & 6 & 11 & 4 & 13 & 454 & 1023 & 59 & 245 & 8 & 38 & 78 & 256 & --- & --- & --- \\
12 &  &  &  &  &  &  &  &  & 460 & 8 & 45 & 90 & 403 & 254 & 1264 & 123 \\
13 &  &  &  &  &  &  &  &  & 165 & 8 & 38 & 39 & 191 & 333 & 1704 & 201 \\
14 &  &  &  &  &  &  &  &  & 423 & 8 & 45 & 90 & 437 & 378 & 1799 & 204 \\
15 &  &  &  &  &  &  &  &  & 179 & 8 & 31 & 38 & 217 & 404 & 1862 & 218 \\
16 &  &  &  &  &  &  &  &  & 246 & 8 & 40 & 52 & 212 & 204 & 843 & 123 \\
17 &  &  &  &  &  &  &  &  & 403 & 8 & 45 & 90 & 372 & 440 & 1903 & 323 \\
18 &  &  &  &  &  &  &  &  & 303 & 8 & 45 & 79 & 333 & 464 & 1905 & 347 \\
19 &  &  &  &  &  &  &  &  & 237 & 8 & 39 & 39 & 179 & 450 & 1885 & 350 \\
20 &  &  &  &  &  &  &  &  & 33 & 8 & 19 & 4 & 36 & 442 & 1841 & 292 \\
21 &  &  &  &  &  &  &  &  & 105 & 8 & 24 & 17 & 100 & 469 & 1775 & 295 \\
22 &  &  &  &  &  &  &  &  & 341 & 8 & 38 & 61 & 202 & 398 & 1589 & 183 \\
23 &  &  &  &  &  &  &  &  & 230 & 8 & 35 & 31 & 138 & 473 & 1527 & 379 \\
\bottomrule
\end{tabular}

\end{table}

\subsection{CelebA Masked Autoencoder}
\label{app:celeba_setup}
\label{app:celeba_profiles}

Our masked autoencoder for CelebA uses $128{\times}128$ images with
patch size~8 and 75\% masking so that 256 of 1024 patches remain visible.
The baseline
(6L/8H, $D{=}384$, $d_k{=}d_v{=}48$, $d_{\mathrm{mlp}}{=}1024$; 12.1M~FMA)
trains without penalty.
The adaptive architecture
(12L/8H, $D{=}768$, $d_k{=}d_v{=}96$, $d_{\mathrm{mlp}}{=}2048$; 90.2M~FMA)
compresses until $\texttt{total\_fma} \leq 12.1\text{M}$.
As on the language task, each run then trains for a further 60k~steps
at frozen structure, annealing to~$10^{-4}$, so the configured
100k~cap does not bind (realized lengths in Table~\ref{tab:exp_setup}).
A follow-up run uses $d_k{=}d_v{=}256$ (145.2M~FMA)
with $L^*{=}0.24$, informed by the $d_k$ saturation observed
in the initial runs.
That condition disables the terminal phase and instead runs the
full 100k~steps with a 30k cosine cooldown (4~seeds).
Batch size~$32{\times}3$ (gradient accumulation),
Muon optimizer with LR~$10^{-3}$ cosine-decayed to~$10^{-4}$,
500-step warmup, weight decay $10^{-2}$ on Adam-optimized
parameters only.
The $d_k{=}256$ run uses more conservative settings
(Table~\ref{tab:exp_setup}).

Figure~\ref{fig:celeba_profiles} shows the structural profiles for all conditions.
The looser target ($L^*{=}0.40$) also shows $d_k$/$d_v$ saturation
but drops structural complexity that the tighter target
($L^*{=}0.32$) retains, illustrating how only a sufficiently
stringent target reveals what the network truly needs.
Across conditions, early attention blocks are significantly
suppressed while MLP complexity stays fairly constant,
suggesting that larger MLPs or embedding dimension~$D$ might be beneficial.
In the $d_k{=}256$ run, $d_k$ and~$d_v$ show a notable ascending trend with depth.

\begin{figure}[p]
\centering
\begin{subfigure}{\textwidth}
\caption{CelebA MAE, $L^*{=}0.40$, $d_k{=}96$ ($9.4{\times}$ compression, loss~0.251).}
\label{fig:celeba_loss040}
\centering
\includegraphics[width=\textwidth]{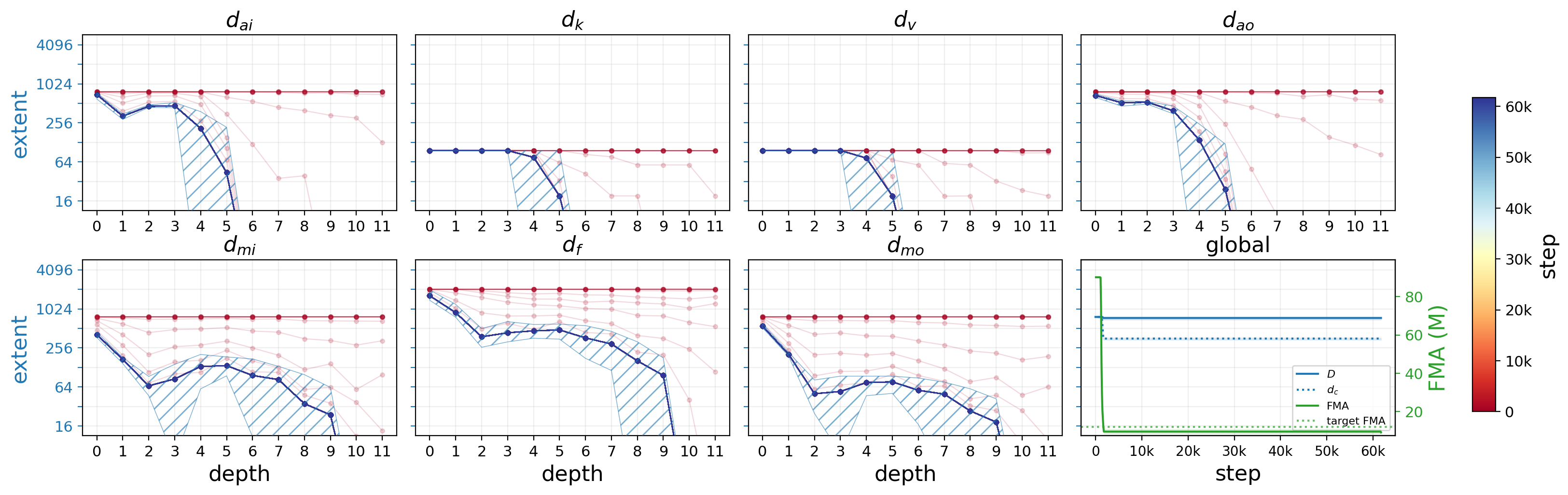}
\end{subfigure}
\begin{subfigure}{\textwidth}
\caption{CelebA MAE, $L^*{=}0.32$, $d_k{=}96$ ($8{\times}$ compression, loss~0.237).}
\label{fig:celeba_loss032}
\centering
\includegraphics[width=\textwidth]{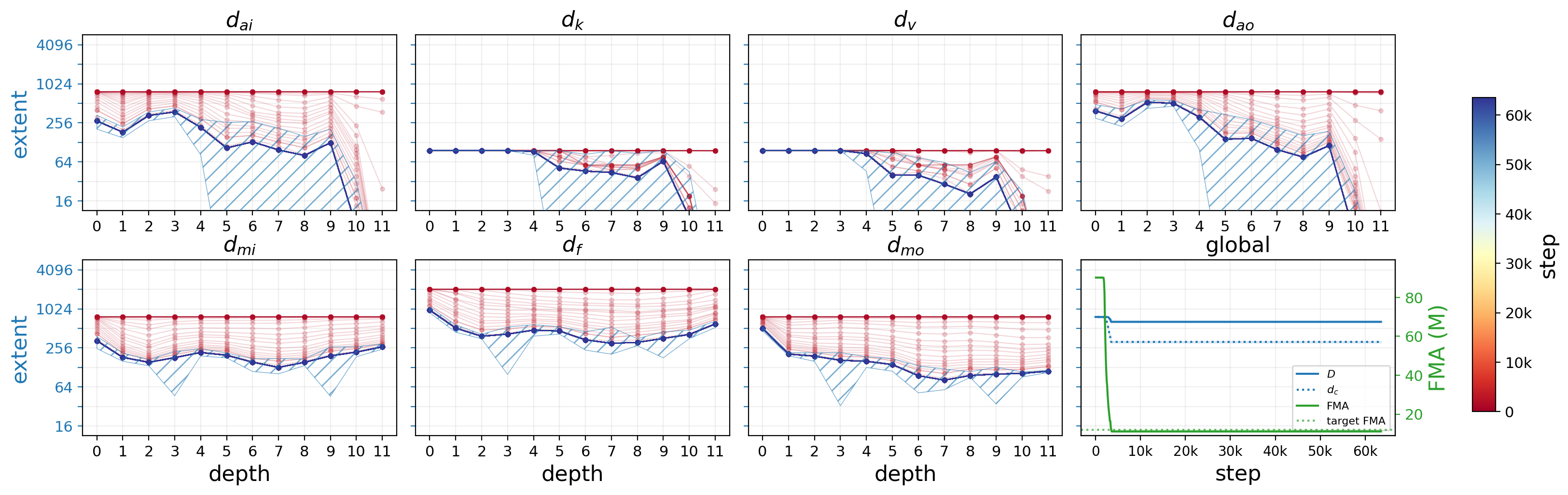}
\end{subfigure}
\begin{subfigure}{\textwidth}
\caption{CelebA MAE, $L^*{=}0.24$, $d_k{=}256$ ($2.9{\times}$ compression, loss~0.213; 4~seeds).}
\label{fig:celeba_dk256}
\centering
\includegraphics[width=\textwidth]{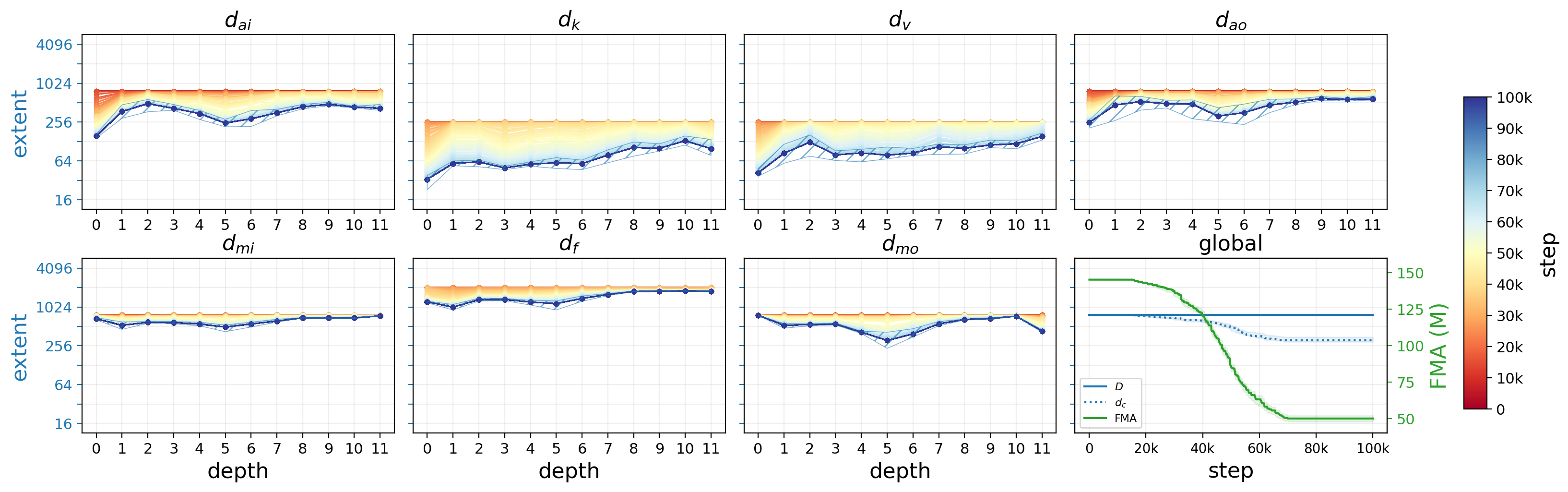}
\end{subfigure}
\caption{CelebA MAE structural profiles.
\textbf{(a)}~The loose target compresses aggressively at the cost of
higher loss.
\textbf{(b)}~The tighter target recovers baseline quality, with all
12~layers surviving.
\textbf{(c)}~$d_k{=}d_v{=}256$, informed by the saturation in~(a)
and~(b), lands below the baseline loss of~0.236.}
\label{fig:celeba_profiles}
\end{figure}

\subsection{Retrieval Task, Ablations, and Structural Profiles}
\label{app:retrieval_setup}
\label{app:ablation_conditions}

The retrieval task uses a vocabulary of 4096~tokens with a
modular-arithmetic index (mod~61) and a sentinel class for
positions where the target is undefined, giving $n_{\mathrm{classes}}
= 4097$.  The model is a causal transformer with RoPE positional
encoding: $D{=}128$, 4~layers, 4~attention heads,
$d_k{=}d_v{=}32$, $d_{\mathrm{mlp}}{=}256$, sequence
length~192, batch size~32.  Training uses the Muon optimizer with
coupled regularization and efficiency-scaled calibration
(Section~\ref{sec:calibration}).  The loss target is
$L^*{=}0.05$, with 30\,000 total steps and a 5\,000-step cooldown
during which the penalty is disabled and no structural changes
occur.  Structural events are spaced at intervals of 100~steps
(trust-drop) and 500~steps (exchange and gauge Sinkhorn).

Each ablation deviates from the baseline in exactly one dimension.
Five seeds (42--46) per condition; 35 runs total.
Table~\ref{tab:ablation-conditions} summarizes the conditions;
full training settings appear in Table~\ref{tab:exp_setup}.

\begin{table}[ht]
\centering
\caption{Ablation conditions.  The baseline uses the conservative
$\rho$-solver, gauge Sinkhorn (before every drop check),
exchange (500-step cadence), and trust multiplier $k{=}1$.
Each row shows the single deviation from baseline.}
\label{tab:ablation-conditions}
\small
\begin{tabular}{lp{7.8cm}}
\toprule
Condition & Deviation from baseline \\
\midrule
Baseline        & --- \\
No Sinkhorn     & Gauge Sinkhorn disabled \\
No Exchange     & Channel exchange disabled (\texttt{exchange\_steps}${}=0$) \\
Sign Step       & $\rho$-solver switched to sign-step \\
Trust $k{=}4$   & Trust multiplier $k$ raised from 1 to 4 \\
Trust $k{=}16$  & Trust multiplier $k$ raised from 1 to 16 \\
\bottomrule
\end{tabular}
\end{table}

Figures~\ref{fig:retriever_part1}--\ref{fig:retriever_part2} show
the per-axis structural profiles discovered by each ablation
condition.

\begin{figure}[p]
\centering
\begin{subfigure}{\textwidth}
\caption{Baseline: Sinkhorn + exchange + conservative solver (203k~FMA).}
\label{fig:retriever_B}
\centering
\includegraphics[width=\textwidth]{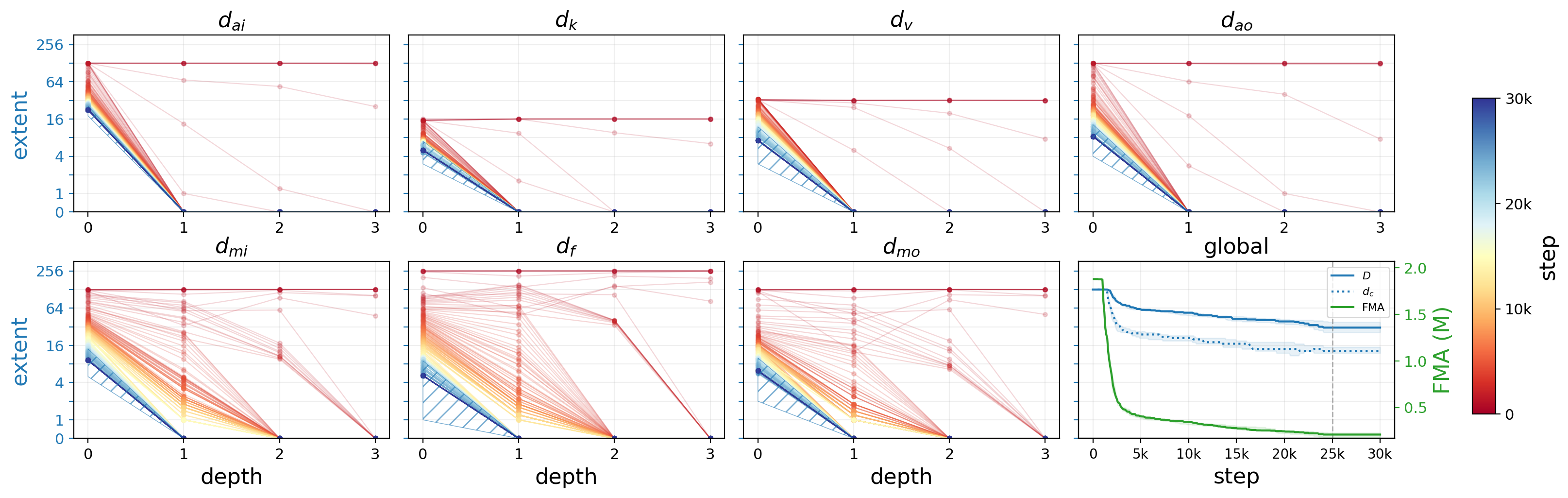}
\end{subfigure}
\begin{subfigure}{\textwidth}
\caption{No gauge Sinkhorn (255k~FMA, $+26\%$).}
\label{fig:retriever_1a}
\centering
\includegraphics[width=\textwidth]{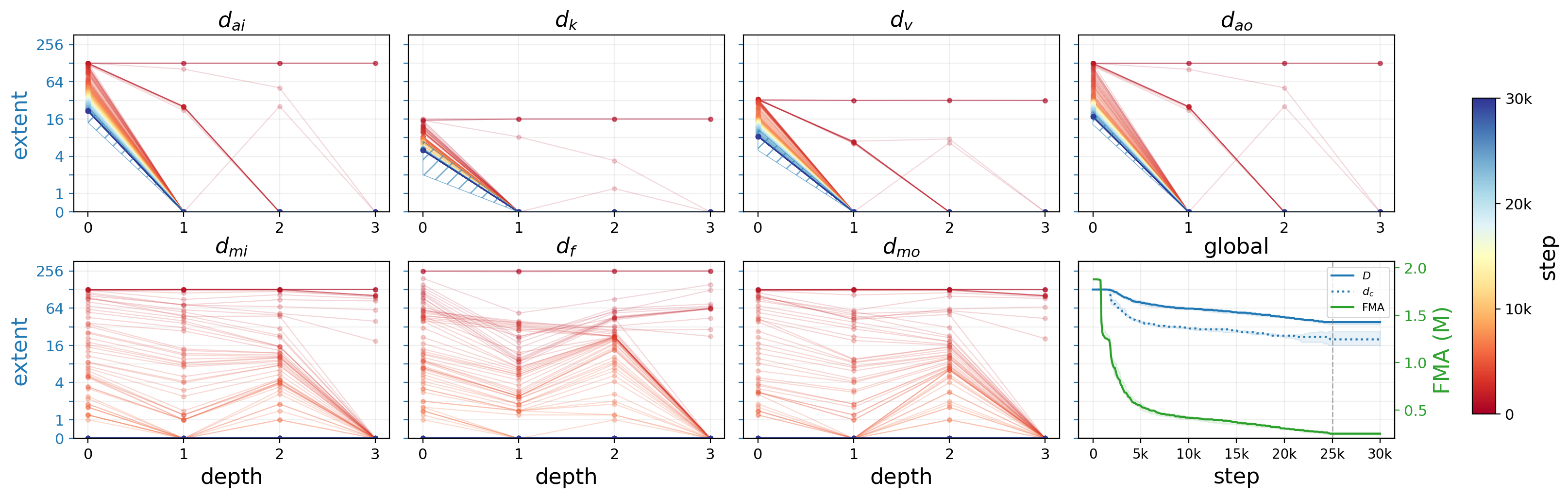}
\end{subfigure}
\begin{subfigure}{\textwidth}
\caption{No exchange (278k~FMA, $+37\%$).}
\label{fig:retriever_1b}
\centering
\includegraphics[width=\textwidth]{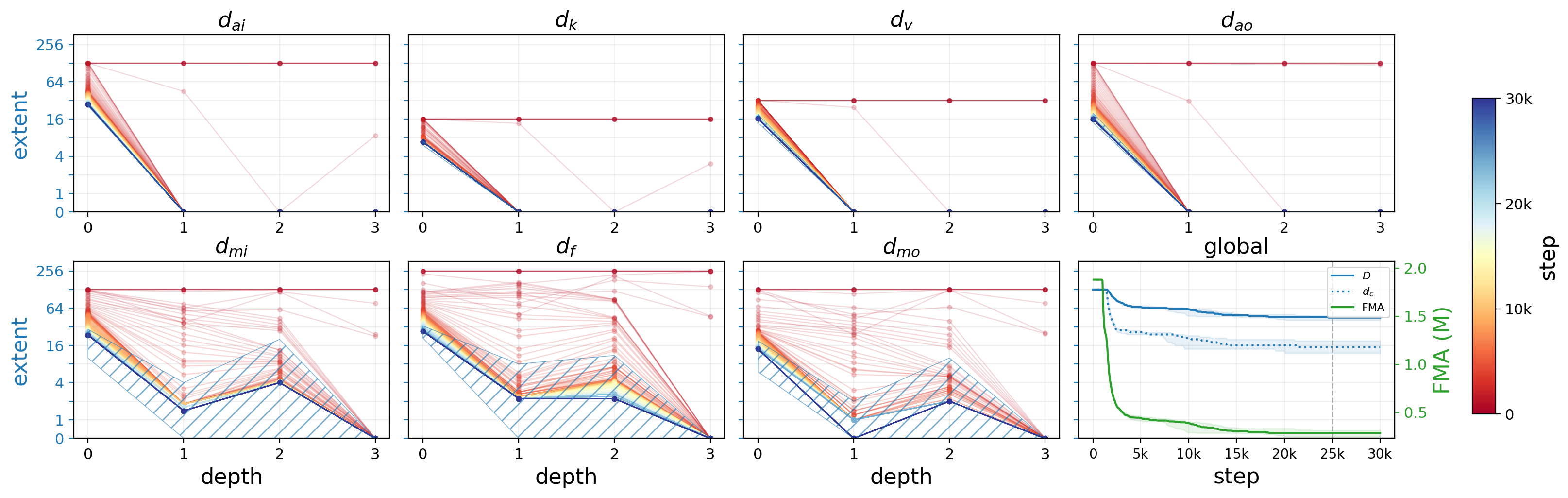}
\end{subfigure}
\caption{Retriever ablation profiles (part~1).
\textbf{(a)}~The baseline compresses $D$ from 128 to ${\sim}30$.
\textbf{(b)}~Without gauge Sinkhorn the optimizer alone cannot expose
channel importance quickly enough.
\textbf{(c)}~Exchange is the most impactful mechanism.}
\label{fig:retriever_part1}
\end{figure}

\begin{figure}[p]
\centering
\begin{subfigure}{\textwidth}
\caption{Sign-step $\rho$-solver (101k~FMA, unstable;
  band over the 4 completed seeds).}
\label{fig:retriever_2a}
\centering
\includegraphics[width=\textwidth]{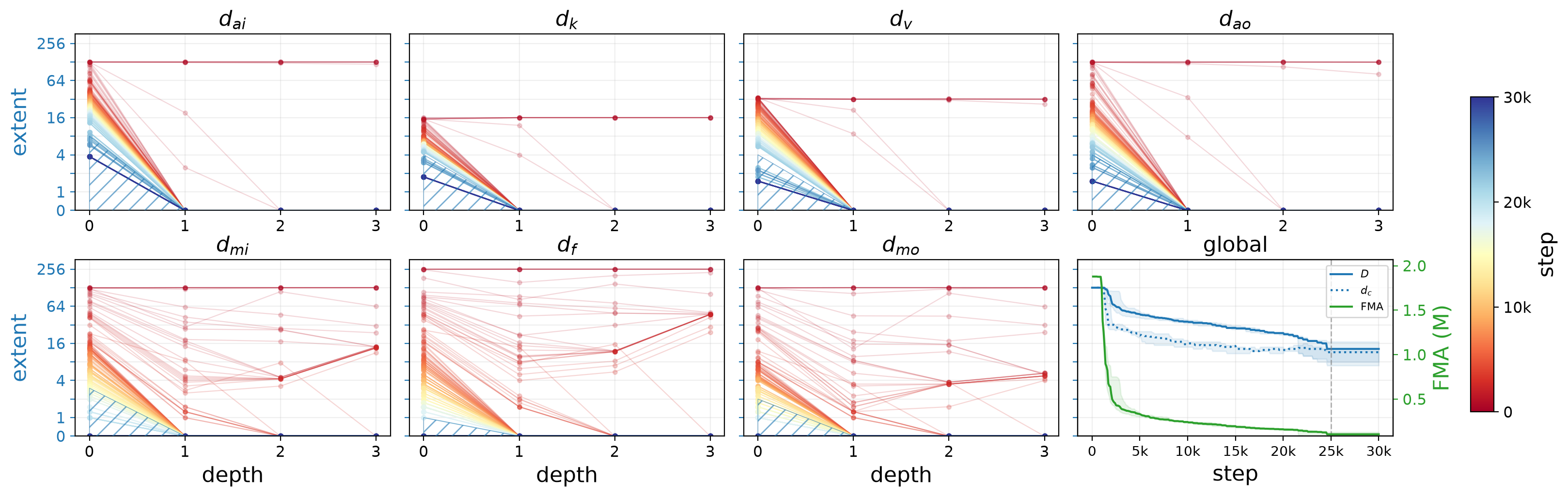}
\end{subfigure}
\begin{subfigure}{\textwidth}
\caption{Trust multiplier $k{=}4$ (199k~FMA).}
\label{fig:retriever_2b}
\centering
\includegraphics[width=\textwidth]{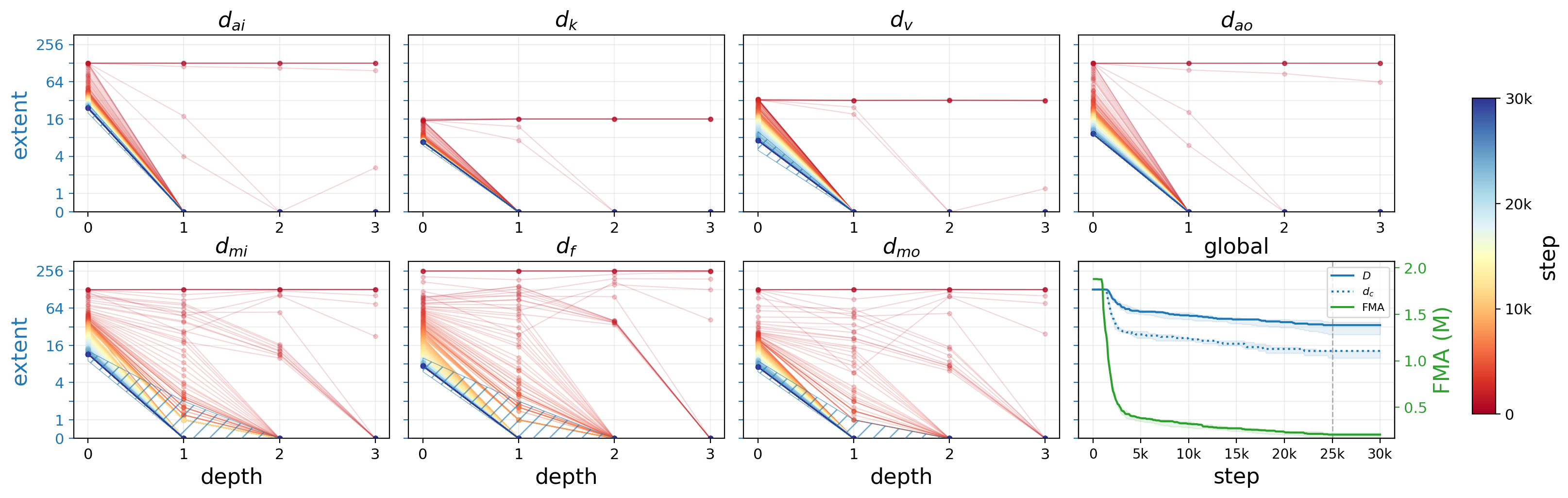}
\end{subfigure}
\begin{subfigure}{\textwidth}
\caption{Trust multiplier $k{=}16$ (206k~FMA).}
\label{fig:retriever_2c}
\centering
\includegraphics[width=\textwidth]{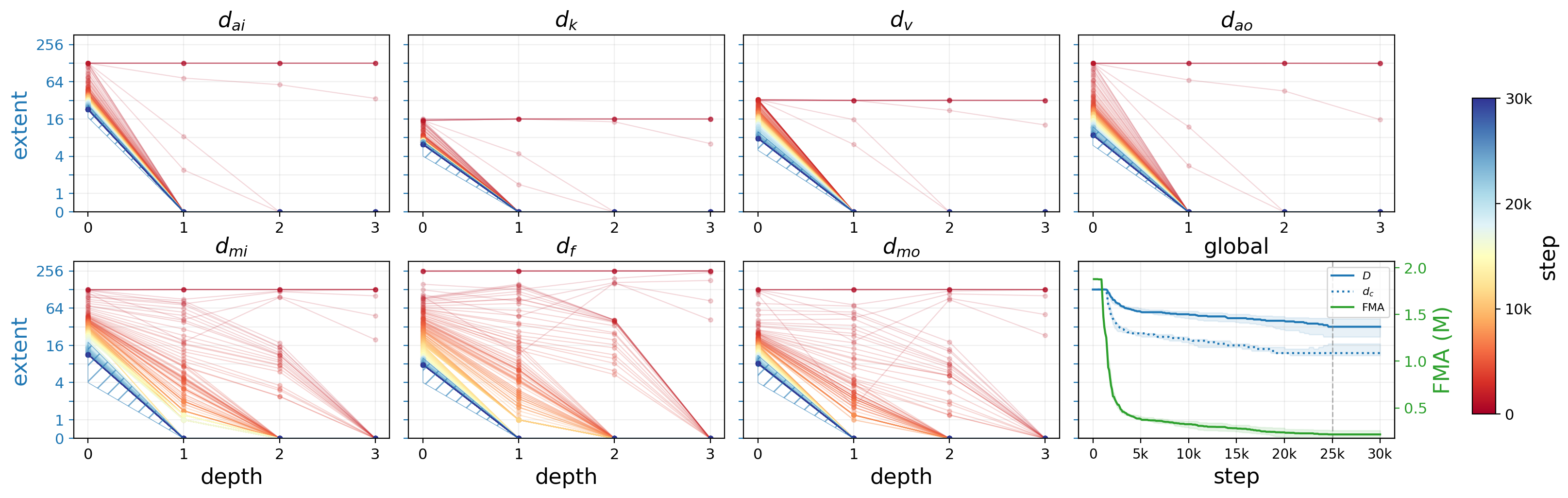}
\end{subfigure}
\caption{Retriever ablation profiles (part~2).
\textbf{(a)}~The sign-step solver over-penalizes and destroys the model
(loss~5.6, accuracy~25\%).
\textbf{(b,\,c)}~Raising the trust multiplier has minimal effect, so the
attractor is robust to the drop threshold.}
\label{fig:retriever_part2}
\end{figure}

\subsection{Post-Hoc Pruning Protocol}
\label{app:posthoc_details}

The post-hoc experiments in Section~\ref{sec:posthoc} follow a
three-phase pipeline using the same gauge-calibrated utility
metric as during-training compression.

\paragraph{FF31.}
Phase~A: train the full 8L/8H/$D{=}192$ model without penalty
for 40k~steps (30k training + 10k cosine cooldown), 5~seeds (42--46).
Phase~B: iteratively apply gauge Sinkhorn ($\tau{=}0.02$, 3~iterations),
remove the lowest-efficiency channel, and train for 10~corrective steps at LR$_{\min}$.
Stop when holdout loss exceeds $L^*{=}0.1$.
Phase~C: finetune for 10k~steps with cosine LR decay, no penalty.
Of 5~seeds, 2 learned the task in Phase~A (${\geq}99\%$ exact match);
the remaining 3 are excluded.

\paragraph{Character-level LM.}
Phase~A: train the full 12L/6H/$D{=}576$ model without penalty
for 10k~steps with cosine LR decay, 5~seeds (42--46).
Phase~B: prune to the baseline FMA target (7.6M) over
10~rounds of 100~corrective steps each (1k~total), dropping
to linearly interpolated FMA targets at each round.
Phase~C: finetune for 36k~steps with cosine LR decay, no penalty.
Total training budget: 47k~steps, matching the during-training runs.

\paragraph{From-scratch controls.}
Each from-scratch arm of Table~\ref{tab:posthoc} takes the structure
one compressed run discovered, re-initializes it, and trains it with
the penalty disabled at 5~seeds (42--46):
FF31 uses seed~43's structure at $L^*{=}0.005$, character-level LM
seed~42's at $L^*{=}\ln 3$, and retrieval seed~45's.
The spread across those five seeds is therefore training-seed variance
at a fixed architecture, not variance over architectures.
Three details qualify the character-level LM arm in particular.
It runs weight decay $0.001$ where both of its comparators ran~$0$;
decay is applied only on the Adam path, so it reaches the embedding,
positional-embedding and RoPE frequency tensors and shrinks them by
$1.8\%$ over the run.
Its realized budget is shorter than its configured cap because the
saturation trigger fires early: a control that starts already at the
target FMA meets both targets almost immediately, at step ${\sim}1.7$k
against ${\sim}16.9$k for its during-training comparator.
Each arm therefore enters an identical $30$k-step anneal --- same peak
rate, same warmup, same floor --- on first meeting the same FMA and
loss targets, rather than on an equal total budget.
Equal total steps would in fact favor the control, giving it
${\sim}15$k extra steps at peak rate on an architecture it never has to
change.
Both mismatches cut the same way, which is why the arm is reported as
consistent with parity rather than as parity.
On FF31, finally, the one control seed that learns the task is seed~43,
the donor of the structure; with a single seed that is an observation
and not a finding.

%
\section{Supplementary Results}
\label{app:supplementary_results}

Appendix~\ref{app:exp_details} is organized by task; this one is
organized by question, and reports results that cut across tasks or go
beyond what the main body has room for.

\subsection{Full-Population Results}
\label{app:full_population}

Table~\ref{tab:ff31}'s FF31 statistics are over converged seeds only.
Table~\ref{tab:ff31_population} lists all 5 seeds at $L^*{=}0.01$ and
all 25 at $L^*{=}0.005$, giving final exact match, final FMA, and
whether each seed met that convergence criterion
(exact match ${\geq}99\%$ and at least $2{\times}$ compression).
The failures are not uniform.
Three seeds --- 46 at the looser target, 42 and 63 at the tighter one
--- end at the initial 13.2M~FMA, having never compressed at all.
The remaining non-converged seeds did compress, into the same
52k--95k band as the converged ones, and finished below the
exact-match criterion.
A mechanism for that second class, and the corrected-schedule
population, are in Appendix~\ref{app:reliability}.

Table~\ref{tab:population} gives per-seed values for every other
reported condition.
Nothing is excluded there: the body tables for these tasks already
aggregate over the full population, and the per-seed values are listed
so the spread behind each mean is visible.
Its final group holds the three from-scratch controls of
Table~\ref{tab:posthoc}, each in its own task's metric.

\begin{table}[ht]
\centering
\caption{FF31 full population.
  Every seed at both loss targets, with the convergence criterion of
  Table~\ref{tab:ff31} applied per seed.
  Seeds ending at 13,180k~FMA never compressed;
  the initial architecture is 13.2M~FMA.}
\label{tab:ff31_population}
\scriptsize
\begin{tabular}{rcccccc}
\toprule
 & \multicolumn{3}{c}{$L^*=0.01$, 40k steps} & \multicolumn{3}{c}{$L^*=0.005$, 70k steps} \\
\cmidrule(lr){2-4}\cmidrule(lr){5-7}
Seed & Exact match & FMA & Conv. & Exact match & FMA & Conv. \\
\midrule
42 & $1.000$ & 63k & \checkmark & $0.025$ & 13180k &  \\
43 & $1.000$ & 63k & \checkmark & $1.000$ & 71k & \checkmark \\
44 & $1.000$ & 90k & \checkmark & $1.000$ & 89k & \checkmark \\
45 & $0.408$ & 95k &  & $1.000$ & 59k & \checkmark \\
46 & $0.031$ & 13180k &  & $1.000$ & 74k & \checkmark \\
47 & --- & --- & --- & $1.000$ & 72k & \checkmark \\
48 & --- & --- & --- & $1.000$ & 56k & \checkmark \\
49 & --- & --- & --- & $0.940$ & 83k &  \\
50 & --- & --- & --- & $1.000$ & 70k & \checkmark \\
51 & --- & --- & --- & $1.000$ & 77k & \checkmark \\
52 & --- & --- & --- & $0.084$ & 69k &  \\
53 & --- & --- & --- & $1.000$ & 70k & \checkmark \\
54 & --- & --- & --- & $0.994$ & 52k & \checkmark \\
55 & --- & --- & --- & $1.000$ & 76k & \checkmark \\
56 & --- & --- & --- & $1.000$ & 67k & \checkmark \\
57 & --- & --- & --- & $1.000$ & 72k & \checkmark \\
58 & --- & --- & --- & $0.035$ & 81k &  \\
59 & --- & --- & --- & $1.000$ & 85k & \checkmark \\
60 & --- & --- & --- & $1.000$ & 52k & \checkmark \\
61 & --- & --- & --- & $1.000$ & 64k & \checkmark \\
62 & --- & --- & --- & $0.045$ & 90k &  \\
63 & --- & --- & --- & $0.002$ & 13180k &  \\
64 & --- & --- & --- & $1.000$ & 81k & \checkmark \\
65 & --- & --- & --- & $0.416$ & 74k &  \\
66 & --- & --- & --- & $1.000$ & 66k & \checkmark \\
\midrule
 & \multicolumn{3}{c}{3 of 5 converged} & \multicolumn{3}{c}{18 of 25 converged} \\
\bottomrule
\end{tabular}

\end{table}

\begin{table}[ht]
\centering
\caption{Full population by seed, for every reported condition outside
  Table~\ref{tab:ff31_population}.
  Perplexity and loss are the values the body tables average;
  retrieval accuracy is read at step~30k, as in
  Table~\ref{tab:retriever-ablation}.
  In the last group all three controls train a single discovered
  architecture across five seeds, so their spread is training-seed
  variance at a fixed architecture
  (Appendix~\ref{app:posthoc_details}).}
\label{tab:population}
\scriptsize
\begin{tabular}{lccccc}
\toprule
Condition & 42 & 43 & 44 & 45 & 46 \\
\midrule
\multicolumn{6}{@{}l}{\emph{Character-level LM (perplexity)}} \\
Baseline & $2.668$ & $2.614$ & $2.657$ & $2.682$ & $2.639$ \\
$L^*{=}\ln 3$ & $2.634$ & $2.614$ & $2.628$ & $2.574$ & $2.599$ \\
$L^*{=}\ln 4$ & $2.731$ & $2.726$ & $2.728$ & $2.881$ & $2.658$ \\
24L, $L^*{=}\ln 2.5$ & $2.317$ & $2.374$ & $2.353$ & $2.307$ & $2.326$ \\
\midrule
\multicolumn{6}{@{}l}{\emph{CelebA masked autoencoding (loss)}} \\
Baseline & $0.236$ & $0.235$ & $0.238$ & $0.235$ & $0.237$ \\
$L^*{=}0.32$ & $0.238$ & $0.236$ & $0.237$ & $0.237$ & $0.238$ \\
$L^*{=}0.40$ & $0.251$ & $0.254$ & $0.251$ & $0.249$ & $0.251$ \\
$L^*{=}0.24$, $d_k{=}256$ & --- & $0.213$ & $0.213$ & $0.213$ & $0.213$ \\
\midrule
\multicolumn{6}{@{}l}{\emph{Retrieval (accuracy at step 30k)}} \\
Baseline & $0.9715$ & $0.9998$ & $0.9999$ & $0.9999$ & $0.9435$ \\
\midrule
\multicolumn{6}{@{}l}{\emph{Discovered architectures trained from fresh initialization}} \\
FF31 (exact match) & $0.002$ & $0.805$ & $0.000$ & $0.000$ & $0.000$ \\
Char LM (perplexity) & $2.670$ & $2.621$ & $2.647$ & $2.656$ & $2.603$ \\
Retrieval (accuracy) & $1.0000$ & $1.0000$ & $0.9996$ & $1.0000$ & $1.0000$ \\
\bottomrule
\end{tabular}

\end{table}

\subsection{Gauge-Component Ablation}
\label{app:gauge_ablation}

This section isolates which results depend on the penalty being
\emph{gauge-correct} rather than merely structured.
Three arms replace the symmetric penalty on the retrieval task,
holding the calibration, the trust-drop criterion and every schedule
fixed: a one-sided penalty, the same with gauge Sinkhorn enabled, and a
product-norm penalty.
Each is the in-vivo counterpart of a pathology argued in
Appendix~\ref{app:failure_modes}, so together they test
Propositions~\ref{prop:equilibrium}--\ref{prop:balance} empirically
rather than by construction.
Results are in Table~\ref{tab:gauge-ablation}.

\paragraph{Read compression, not accuracy.}
The one-sided arm reaches the \emph{highest} accuracy in the table,
$0.996$, and that is the failure rather than a success: it gets there by
not compressing.
At $442$k~FMA it sits $1.7{\times}$ above the matched symmetric control,
having converged to an architecture the penalty was supposed to shrink.
This is the gauge problem of Appendix~\ref{app:failure_modes} made
concrete: the network preserves its function by moving magnitude into
the unpenalized factor, so the measured penalty falls without any
structure being given up.
Charging both factors removes the escape
(Proposition~\ref{prop:balance}).

\paragraph{Gauge Sinkhorn cannot repair a one-sided penalty.}
The cleanest contrast in the table is between the two one-sided arms,
which differ in one flag and share seeds, schedule and every other
setting.
Enabling gauge Sinkhorn moves compression only from $442$k to $302$k
while accuracy collapses from $0.996$ to $0.368$.
Rebalancing and one-sided pressure work against each other: Sinkhorn
restores the balance that the penalty is simultaneously destroying, and
the network is driven through repeated rescalings rather than toward a
structure.
Gauge correctness is therefore a property of the penalty, not something
a rebalancing step can add afterwards.

\paragraph{Product-norm over-compresses past the point of quality.}
The product-norm arm compresses hardest --- $150$k~FMA, $0.6{\times}$
the control --- and loses the task doing it, at accuracy
$0.579 \pm 0.308$.
Appendix~\ref{app:failure_modes} predicts this: the product-norm
gradient on one factor grows without bound as the other factor shrinks,
so once a channel starts to go it is driven out faster than the loss can
object.
The spread is the diagnostic, not the mean --- the arm is bimodal across
seeds rather than uniformly worse.

Taken together, the symmetric penalty is the only one of the three that
compresses \emph{and} keeps the task, and the two failures are opposite:
one is evaded, the other unbounded.

\paragraph{Scope.}
The variant arms vary the data stream with the seed while the symmetric
reference runs vary only initialization.
The two generations only share a schedule through step~25{,}000 (when
the penalty freezes), which is why the table reports at step~25k
throughout.
The one-sided-with-Sinkhorn arm has no exactly matched control, since no
symmetric run used a $500$-step Sinkhorn cadence; the one-sided pair it
is contrasted against is exact.

\begin{table}[ht]
\centering
\caption{Gauge-component ablation on the retrieval task, at step~25k
  (mean $\pm$ std over 5~seeds).
  FMA ratios are against the matched symmetric control --- symmetric
  penalty, gauge Sinkhorn off --- which is the second row.
  Accuracy is reported before the cooldown, so it is lower than the
  converged figures in Table~\ref{tab:retriever-ablation};
  the comparison across arms is the point, not the level.}
\label{tab:gauge-ablation}
\small
\begin{tabular}{@{}llccc@{}}
\toprule
Penalty & Gauge Sinkhorn & Accuracy & Loss & FMA \\
\midrule
\multicolumn{5}{@{}l}{\emph{Symmetric reference runs}} \\
Symmetric & gS every 100 & $0.765 \pm 0.292$ & $1.657 \pm 1.981$ & 203k (0.8$\times$) \\
Symmetric & no gS & $0.980 \pm 0.040$ & $0.266 \pm 0.491$ & 255k \\
\midrule
\multicolumn{5}{@{}l}{\emph{Gauge-variant arms}} \\
One-sided & no gS & $0.996 \pm 0.007$ & $0.037 \pm 0.057$ & 442k (1.7$\times$) \\
One-sided & gS every 500 & $0.368 \pm 0.265$ & $5.265 \pm 1.913$ & 302k (1.2$\times$) \\
Product-norm & no gS & $0.579 \pm 0.308$ & $2.894 \pm 2.117$ & 150k (0.6$\times$) \\
\bottomrule
\end{tabular}

\end{table}

\subsection{An External During-Training Method}
\label{app:cofi}

Section~\ref{sec:posthoc} reports that our reimplementation of
CoFi~\citep{xia2022structured} stalls well short of the requested compression.
This section details the two runs and the controller state that explain
this outcome.

Both runs distill from the same teacher,
the full 12-layer character LM of Appendix~\ref{app:lm_setup} trained without penalty,
and both follow the same schedule of $5$k warmup, $40$k pruning and
$25$k finetuning steps;
they differ only in the measured resource objective.
One targets the FMA of our own compressed model, the other its parameter count.
Table~\ref{tab:cofi} reports each request and what it reached,
each in its own objective.
The parameter target misses by as much as the FMA target does,
at $2.05{\times}$ against $8.38{\times}$ requested.
Quality is not the binding constraint:
both arms finish within $0.1$ perplexity of the teacher.

Table~\ref{tab:cofi_controller} reads the controller at the last step of the pruning phase.
The coarse mask families (whole attention blocks, residual channels, and whole MLP blocks)
sit at a mean of $1.000$, meaning nothing at that granularity was ever given up.
All of the realized compression came from the two fine families.
After the full $40$k pruning steps both Lagrange multipliers are still
climbing at roughly ten units per thousand steps.
The phase ended with the controller still pushing, not at an equilibrium.
The last column of Table~\ref{tab:cofi} records the consequence for shape:
head width is untouched in all twelve sublayers.
The architectures we reach narrow key and value dimensions \emph{inside}
heads that stay alive (Table~\ref{tab:per_layer}), which the published mask
space cannot express.
This comparison therefore bounds what is reachable through those mask
families, not what the method could do with others.

\begin{table}[ht]
\centering
\caption{CoFi against the same compression target, one seed per arm.
  Parameter counts exclude the embedding tables, as the controller's own
  accounting does.
  Perplexity is held out, so it is not comparable with the last-report
  values used elsewhere.
  \emph{Full width} counts sublayers whose $d_k$ and $d_v$ still match
  the teacher's.}
\label{tab:cofi}
\small
\begin{tabular}{@{}lcccccc@{}}
\toprule
Model & Requested & Reached & FMA/tok & Parameters & Perplexity & Full width \\
\midrule
Teacher (no pruning) & --- & $1.00\times$ & 53.7M & 48.0M & $2.527$ & 12 of 12 \\
CoFi, FMA target & $7.09\times$ & $1.97\times$ & 27.3M & 23.5M & $2.622$ & 12 of 12 \\
CoFi, parameter target & $8.38\times$ & $2.05\times$ & 27.5M & 23.3M & $2.625$ & 12 of 12 \\
\bottomrule
\end{tabular}

\end{table}

\begin{table}[ht]
\centering
\caption{CoFi controller state at the last step of the pruning phase.
  Mask values are means per family: MHA and FFN gate whole sublayers,
  hidden the residual channels, head the attention heads, intermediate
  the MLP units.
  $\Delta\lambda_1$ is the change in the first multiplier over the final
  $5{,}000$ steps, per thousand steps.}
\label{tab:cofi_controller}
\small
\begin{tabular}{@{}lcccccccc@{}}
\toprule
 & \multicolumn{5}{c}{Mean mask value by family} & \multicolumn{3}{c}{Lagrange multipliers} \\
\cmidrule(lr){2-6}\cmidrule(lr){7-9}
Arm & MHA & Head & Hidden & FFN & Interm. & $\lambda_1$ & $\lambda_2$ & $\Delta\lambda_1$/1k \\
\midrule
CoFi, FMA target & $1.000$ & $0.639$ & $1.000$ & $1.000$ & $0.412$ & $369$ & $389$ & $+9.9$ \\
CoFi, parameter target & $1.000$ & $0.694$ & $1.000$ & $1.000$ & $0.380$ & $368$ & $389$ & $+9.8$ \\
\bottomrule
\end{tabular}

\end{table}

\subsection{Reliability and Schedule Sensitivity}
\label{app:reliability}

Compression pressure is gated on the loss target: until the loss reaches
$L^*$ the penalty never engages and no channel is removed.
On FF31 the model passes a plateau before reaching the target, so the
time taken to escape it decides whether compression begins at all.
This section characterizes that coupling and reports the population at
the corrected schedule.

To separate the schedule from the method we repeated $L^*{=}\ffrelTarget$
at half the learning rate over the same \ffrelSeeds~seeds.
Table~\ref{tab:ff31_reliability} gives both populations.
These trials reveal two modes of failure, both related to the loss target.
A seed that never reaches $L^*$ cannot compress, thus ending training
at its initial FMA.
A seed may also begin compressing normally, but then fail to maintain the loss target
or pass the exact match condition.
At $\eta{=}10^{-3}$, \ffrelPubNeverComp~seeds fail to start and
\ffrelPubDegraded~fail to hold; at $\eta{=}5{\times}10^{-4}$ all seeds begin compression
and \ffrelHalfDegraded~seeds still fail to hold.

On this architecture, halving the learning rate results in more efficient
early descent and, thus, earlier compression.
The median falls from \ffrelPubReachMedian\ steps to \ffrelHalfReachMedian,
and the range from [\ffrelPubReachMin, \ffrelPubReachMax]
with \ffrelPubNeverReached~seeds never starting,
to [\ffrelHalfReachMin, \ffrelHalfReachMax] with no lost starts
($p\,\ffrelReachP$).
The fail-to-hold count does not change ($p{=}\ffrelDegradedFisherP$).
Convergence rises slightly from \ffrelPubConverged\ to \ffrelHalfConverged\ of
\ffrelSeeds, a difference these populations cannot resolve
($p{=}\ffrelRateFisherP$) and which we do not read as a reliability gain.
Further, the converged runs achieve lower compression at the smaller learning rate,
\ffrelPubCompression$\times$ against \ffrelHalfCompression$\times$ ($p{=}\ffrelConvFMAP$).

These experiments highlight how compression is a dynamical process whose
outcome depends on everything that governs when and the extent to which
the loss falls below the target, including initialization, data order,
and learning rate. 
FF31 is sensitive because of the very strict loss target and accuracy criterion.

\begin{table}[htbp]
\centering
\caption{FF31 reliability at $L^*{=}\ffrelTarget$, the same \ffrelSeeds~seeds
  at two learning rates.
  Compression is the initial \ffrelInitialFMA~FMA over each class's median
  final FMA, with $[\min,\max]$ over the class.
  Convergence is exact match ${\geq}\,0.99$ with at least $2\times$
  compression, as in Table~\ref{tab:ff31}.}
\label{tab:ff31_reliability}
\scriptsize
\begin{tabular}{@{}lcccccc@{}}
\toprule
 & \multicolumn{3}{c}{$\eta = 10^{-3}$} & \multicolumn{3}{c}{$\eta = 5 \times 10^{-4}$} \\
\cmidrule(lr){2-4}\cmidrule(lr){5-7}
Outcome & Runs & Compression & Final loss & Runs & Compression & Final loss \\
\midrule
Never compressed & 2/25 & $1\times$ & $0.0811$--$0.0949$ & 0/25 & --- & --- \\
Compressed, converged & 18/25 & $187\times$ [148, 255] & $\leq 0.0029$ & 20/25 & $176\times$ [129, 208] & $\leq 0.0048$ \\
Compressed, not converged & 5/25 & $162\times$ [146, 191] & $0.0004$--$0.0132$ & 5/25 & $193\times$ [148, 224] & $0.0038$--$0.0163$ \\
\midrule
Steps to first reach $L^*$ & \multicolumn{3}{c}{2400 [1700, 17100]; 2 never} & \multicolumn{3}{c}{1700 [1200, 3300]; 0 never} \\
Accepted by final loss $\leq L^*$ & \multicolumn{3}{c}{19 accepted, 18 converged} & \multicolumn{3}{c}{21 accepted, 20 converged} \\
Non-converged seeds & \multicolumn{3}{c}{42, 63 (never); 49, 52, 58, 62, 65} & \multicolumn{3}{c}{51, 57, 60, 62, 63} \\
\bottomrule
\end{tabular}

\end{table}

\end{document}